\documentclass[letterpaper, 10 pt]{IEEEtran}
\usepackage{textcomp}
\usepackage{color}
\usepackage{float}
\usepackage{multirow}
\usepackage{amsmath}
\usepackage{amssymb,eqnarray,amsthm}
\usepackage{mathabx,mathtools}
\usepackage{graphicx}
\usepackage{stmaryrd}
\usepackage{psfrag}
\usepackage{pstool}

\usepackage{booktabs} % nicer tables :-)
\usepackage{multirow}   % afor multi-row tables
\usepackage{array, makecell} %

\DeclareMathOperator*{\essinf}{ess\,inf}

 \DeclareMathOperator*{\argmin}{argmin}

\makeatletter
\newcommand{\pushright}[1]{\ifmeasuring@#1\else\omit\hfill$\equationstyle#1$\fi\ignorespaces}
\newcommand{\pushleft}[1]{\ifmeasuring@#1\else\omit$\equationstyle#1$\hfill\fi\ignorespaces}
\makeatother

\DeclareMathOperator*{\esssup}{ess\,sup}

\usepackage{multicol}
\usepackage{bbm}
\usepackage[noend]{algorithmic}
\usepackage{algorithm2e}
\usepackage[english]{babel}
\usepackage{subfigure}
\usepackage{cite}
\usepackage{mathtools}
\mathtoolsset{showonlyrefs=false}
\makeatletter

\floatstyle{ruled}
\newfloat{algorithm}{tbp}{loa}
\providecommand{\algorithmname}{Algorithm}
\floatname{algorithm}{\protect\algorithmname}

\newtheorem{theorem}{\protect\theoremname}
\newtheorem{defn}{\protect\definitionname}
\newtheorem{proposition}{\protect\propositionname}

\newtheorem{assum}{\protect\assumname}
\newtheorem{problem}{\protect\probname}
\newtheorem{corollary}{\protect\corname}

\providecommand{\definitionname}{\textbf{Definition}}
\providecommand{\propositionname}{\textbf{Proposition}}
\providecommand{\remarkname}{\textbf{Remark}}
\providecommand{\theoremname}{\textbf{Theorem}}
\providecommand{\lemmaname}{Lemma}
\providecommand{\assumname}{\textbf{Assumption}}
\providecommand{\probname}{\textbf{Problem}}
\providecommand{\corname}{\textbf{Corollary}}

\begin{document}

\title{Risk-Averse Decision Making Under Uncertainty}
% \title{Risk-Averse Planning Under Uncertainty:\\ POMDPs with Coherent Risk Objectives}
%\author{Mohamadreza Ahmadi, Richard M. Murray, and Aaron D. Ames}
\author{Mohamadreza Ahmadi, Ugo Rosolia, 
Michel D. Ingham, Richard M. Murray, and Aaron D. Ames \thanks{M. Ahmadi, U. Rosolia, R. Murray, and A. Ames are with Control and Dynamical Systems (CDS) at the California Institute of Technology, 1200 E. California Blvd., MC 104-44, Pasadena, CA 91125,  e-mail: (\{mrahmadi,urosolia,murray,ames\}@caltech.edu). M. Ingham is with NASA Jet Propulsion Laboratory, 4800 Oak Grove Dr, Pasadena, CA 91109, e-mail: (michel.d.ingham@jpl.nasa.gov).}}
%\date{March 2019}

\maketitle

\begin{abstract}
A large class of decision making under uncertainty problems can be described via Markov decision processes (MDPs) or partially observable MDPs (POMDPs), with application to artificial intelligence and operations research, among others. Traditionally, policy synthesis techniques are proposed such that a total expected cost/reward is minimized/maximized. However, optimality in the total expected cost sense is only reasonable if system’s  behavior in the large number
of runs is of interest, which has limited the use of such policies in practical mission-critical scenarios, wherein large deviations from the expected behavior may lead to mission failure. In this paper, we consider the problem of designing policies for MDPs and POMDPs with objectives and constraints in terms of dynamic coherent risk measures, which we refer to as the \emph{constrained risk-averse problem}. Our contributions are fourfold:

\begin{itemize}
    \item [(i)] For MDPs, we reformulate the problem into a inf-sup problem via the Lagrangian framework. Under the assumption that the risk objectives and constraints can be represented by a Markov risk transition mapping, we propose an optimization-based method to synthesize Markovian policies;
    \item [(ii)] For MDPs, we demonstrate that the formulated optimization problems are in the form of difference convex programs (DCPs) and can be solved by the disciplined convex-concave programming (DCCP) framework. We show that these results generalize linear programs for constrained MDPs with total discounted expected costs and constraints;
    \item [(iii)] For POMDPs, we show that, if the coherent risk measures can be defined as a Markov risk transition mapping, an infinite-dimensional optimization can be used to design Markovian belief-based policies;
    \item [(iv)] For POMDPs with stochastic finite-state controllers (FSCs), we show that the latter optimization simplifies to a (finite-dimensional) DCP and can be solved by the DCCP framework. We incorporate these DCPs in a policy iteration algorithm to design risk-averse FSCs for POMDPs.
\end{itemize}
We demonstrate the efficacy of the proposed method with numerical experiments involving conditional-value-at-risk (CVaR) and entropic-value-at-risk (EVaR)  risk measures.
\end{abstract}


\section{Introduction}

Autonomous systems are being increasingly deployed in real-world settings. Hence, the associated risk that stems from unknown and unforeseen circumstances is correspondingly on the rise. This demands for autonomous systems that can make appropriately conservative decisions when faced with uncertainty in their environment and behavior. Mathematically speaking, risk can be quantified in numerous ways, such as chance constraints~\cite{wang2020non} and distributional robustness~\cite{NIPS2010_19f3cd30}. However, applications in autonomy and robotics require more ``nuanced assessments of risk''~\cite{majumdar2020should}. Artzner \textit{et. al.}~\cite{artzner1999coherent} characterized a set of natural properties that are desirable for a risk measure, called a coherent risk measure, and  have  obtained widespread
acceptance in finance and operations research, among other fields.

A popular model for representing sequential decision making under uncertainty is a Markov decision processes (MDP)~\cite{Puterman94}. MDPs with coherent risk objectives were studied in~\cite{tamar2016sequential,tamar2015policy}, where the authors proposed a sampling-based algorithm for finding saddle point solutions using policy gradient methods. However, \cite{tamar2016sequential} requires the risk envelope appearing in the dual representation of the coherent risk measure to be known with an explicit canonical convex programming formulation. While this may be the case for CVaR, mean-semi-deviation, and spectral risk measures~\cite{shapiro2014lectures}, such explicit form is not known for  general coherent risk measures, such as EVaR. Furthermore, it is not clear whether the saddle point solutions are a lower bound or upper bound to the optimal value. Also, policy-gradient based methods require calculating the gradient of  the coherent risk measure, which is not available in explicit form in general. For the CVaR measure, MDPs with risk constraints and total expected costs were studied in~\cite{prashanth2014policy,chow2014algorithms} and locally optimal solutions were found via policy gradients, as well. However, this method also leads to saddle point solutions (which cannot be shown to be upper bounds or lower bounds of the optimal value) and cannot be applied to general coherent risk measures. In addition, because the objective and the constraints are in terms of different coherent risk measures, the authors assume there exists a policy that satisfies the CVaR constraint (feasibility assumption), which may not be the case in general. Following the footsteps of~\cite{pflug2016time}, a promising approach based on approximate value iteration was proposed for MDPs with CVaR objectives in~\cite{chow2015risk}. A policy iteration algorithm for finding policies that minimize total coherent risk measures for MDPs was studied in~\cite{ruszczynski2010risk} and a computational non-smooth Newton method was proposed in~\cite{ruszczynski2010risk}.

% Recently,~\cite{ahmadi2021aaai} extended~\cite{ruszczynski2010risk,fan2018process} to MDPs with both coherent risk measure objectives and constraints and proposed a DCCP computational method.

When the states of the agent and/or the environment are not directly observable, a partially observable MDP (POMDP) can be used to study decision making under uncertainty introduced by the partial state observability~\cite{krishnamurthy2016partially,ahmadi2020control}. POMDPs with coherent risk measure objectives were studied in~\cite{fan2018process,fan2018risk}. Despite the elegance of the theory, no computational method was proposed to design policies for general coherent risk measures. In \cite{ahmadi2020risk}, we proposed a method for finding finite-state controllers for POMDPs with objectives defined in terms of coherent risk measures, which takes advantage of convex optimization techniques. However, the method can only be used if the risk transition mapping is affine in the policy. 

\textit{Summary of Contributions:} In this paper, we consider MDPs and POMDPs with both objectives and constraints in terms of coherent risk measures. Our contributions are fourfold:
\begin{itemize}
    \item [(i)] For MDPs, we use the Lagrangian framework and reformulate the problem into a inf-sup problem. For Markov risk transition mappings, we propose an optimization-based method to design Markovian policies that lower-bound the constrained risk-averse problem;
    \item [(ii)] For MDPs, we evince that the optimization problems are in the special form of DCPs and can be solved by the DCCP method. We also demonstrate that these results generalize linear programs for constrained MDPs with total discounted expected costs and constraints;
    \item [(iii)] For POMDPs, we demonstrate that, if the coherent risk measures can be defined as a Markov risk transition mapping, an infinite-dimensional optimization can be used to design Markovian belief-based policies, which in theory requires infinite memory to synthesize (in accordance with classical POMDP complexity results);
    \item [(iv)] For POMDPs with stochastic finite-state controllers (FSCs), we show that the latter optimization converts to a (finite-dimensional) DCP and can be solved by the DCCP framework. We incorporate these DCPs in a policy iteration algorithm to design risk-averse FSCs for POMDPs.
\end{itemize}
We assess the efficacy of the proposed method with numerical experiments involving conditional-value-at-risk (CVaR) and entropic-value-at-risk (EVaR)  risk measures.

Preliminary results on risk-averse MDPs were presented in~\cite{ahmadi2021aaai}. This paper, in addition to providing detailed proofs and new numerical analysis in the MDP case, generalizes~\cite{ahmadi2021aaai} to partially observable systems (POMDPs) with dynamic coherent risk objectives and constraints.

{The rest of the paper is organized as follows. In the next section, we briefly review some notions used in the paper. In Section~III, we formulate the problem under study. In Section IV, we present the optimization-based method for designing risk-averse policies for MDPs. In Section V, we describe a policy iteration method for designing finite-memory controllers for risk-averse POMDPs. In Section VI, we illustrate the proposed methodology via numerical experiments. Finally, in Section~VII, we conclude the paper and give directions for future research.}

\textbf{Notation: } We denote by $\mathbb{R}^n$ the $n$-dimensional Euclidean space and $\mathbb{N}_{\ge0}$ the set of non-negative integers. Throughout the paper, we use bold font to denote a vector and $(\cdot)^\top$ for its transpose, \textit{e.g.,} $\boldsymbol{a}=(a_1,\ldots,a_n)^\top$, with $n\in \{1,2,\ldots\}$. For a vector $\boldsymbol{a}$, we use $\boldsymbol{a}\succeq (\preceq) \boldsymbol{0}$ to denote element-wise non-negativity (non-positivity) and $\boldsymbol{a}\equiv \boldsymbol{0}$ to show all elements of $\boldsymbol{a}$ are zero. For two vectors $a,b \in \mathbb{R}^n$, we denote their inner product by $\langle \boldsymbol{a}, \boldsymbol{b} \rangle$, \textit{i.e.,} $\langle \boldsymbol{a}, \boldsymbol{b} \rangle=\boldsymbol{a}^\top \boldsymbol{b}$. For a finite set $\mathcal{A}$, we denote its power set by $2^\mathcal{A}$, \textit{i.e.,} the set of all subsets of $\mathcal{A}$. For  a probability space $(\Omega, \mathcal{F}, \mathbb{P})$ and a constant $p \in [1,\infty)$, $\mathcal{L}_p(\Omega, \mathcal{F}, \mathbb{P})$ denotes the vector space of real valued random variables $c$ for which $\mathbb{E}|c|^p < \infty$.

\section{Preliminaries}

In this section, we briefly review some notions and definitions used throughout the paper.

\subsection{Markov Decision Processes} \label{sec:POMDP}

%\begin{defn}[MDP] \label{defn:MDP} {
An \emph{MDP} is a tuple $\mathcal{M}=(\mathcal{S},Act, T, \kappa_0)$ consisting of a set of states $\mathcal{S}  =
\{s_{1} ,\dots,s_{|\mathcal{S}|} \}$ of the
autonomous agent(s) and world model, actions $Act = \{\alpha_{1},\dots,\alpha_{|Act|}\}$ available to the agent,
% \item Observations $\mathcal{O} = \{o_{1},\dots,o_{|\mathcal{O}|}\}$,
%\item Atomic propositions $AP = \{p_1,p_2,\dots p_{|AP|}\}$,
a transition function $T(s_{j} |s_{i} ,\alpha)$, and $\kappa_0$ describing the initial distribution over the states. 

This paper considers {\em finite} Markov decision processes, where $\mathcal{S} $ and $Act$
are finite sets. For each action the probability of making a transition from state $s_{i}  \in
\mathcal{S} $ to state $s_{j}  \in \mathcal{S} $ under action $\alpha \in Act$ is given by
$T(s_{j} |s_{i} ,\alpha)$. The probabilistic components of a
MDP must satisfy the following:
\begin{equation*}
    \begin{cases}
    \sum_{s  \in \mathcal{S} } T(s |s_{i} ,\alpha) = 1, & \forall s_i  \in \mathcal{S} ,\forall\alpha \in Act, \\
    \sum_{s  \in \mathcal{S} } \kappa_0(s ) = 1. & {}
    \end{cases}
\end{equation*}

\subsection{Partially Observable MDPs} \label{sec:POMDP}

%\begin{defn}[MDP] \label{defn:MDP} {
A \emph{POMDP} is a tuple $\mathcal{PM}=(\mathcal{M}, \mathcal{O}, O)$ consisting of an MDP $\mathcal{M}$, observations $\mathcal{O} = \{o_{1},\dots,o_{|\mathcal{O}|}\}$,
%\item Atomic propositions $AP = \{p_1,p_2,\dots p_{|AP|}\}$,
and an observation model $O(o\mid s)$. We consider {\em finite} POMDPs, where $\mathcal{O}$ 
is a finite set. Then, for each state $s_{i} $, an observation $o \in
\mathcal{O}$ is generated independently with probability $O(o|s_{i} )$, which satisfies
\begin{equation*}
    \sum_{s  \in \mathcal{S} } O(o |s ) = 1,  \quad \forall s  \in \mathcal{S}.
\end{equation*}

In POMDPs, the states $s\in \mathcal{S}$ are not directly observable.  The beliefs $b \in \Delta(\mathcal{S})$, i.e., the probability of being in different states, with $\Delta(\mathcal{S})$ being the set of probability distributions over $\mathcal{S}$, for all $s \in \mathcal{S}$ can be computed using the Bayes' law as follows:
\begin{align}
    b_0(s) &= \frac{\kappa_0(s)O(o_0\mid s)}{\sum_{o \in O} \kappa_0(s) O(o \mid s)},\\ \label{eq:beliefupdate}
    b_t(s) &= \frac{O(o_t \mid s)\sum_{s' \in \mathcal{S}} T(s \mid s',\alpha_t)b_{t-1}(s')}{\sum_{s \in \mathcal{S}} O(o_t \mid s)\sum_{s' \in \mathcal{S}} T(s \mid s',\alpha_t)b_{t-1}(s')},
\end{align}
for all $t\ge 1$.

\subsection{Finite State Control of POMDPs} \label{sec:FSC}

% \todo{Overview of belief space controller. 
% It is a well known fact that POMDP, and for some criteria, MDP controllers require memory or
% internal states \cite{CassandraKL94, KLC98, AberdeenThesis}. The most popular method that employs
% infinite memory design controllers that work in the belief space which is continuous, which
% effectively implies uncountably infinite internal memory for the controller}.

It is well established that designing optimal policies for POMDPs based on the (continuous) belief states require uncountably infinite memory or
internal states \cite{CassandraKL94,MADANI20035}. This paper focuses on a particular class of POMDP controllers, namely, FSCs.

 A \emph{stochastic finite state controller } for
$\mathcal{PM}$ is given by the tuple $\mathcal{G} = (G,\omega,\kappa)$, where $G = \{g_1,g_2,\dots,g_{|G|}\}$ is a finite set of internal states~(I-states), $\omega:G \times \mathcal{O} \to \Delta({G \times Act})$ is a function of internal stochastic finite state controller states  $g_k$ and observation $o$, such that $\omega(g_k,o)$ is a probability distribution over $G \times
Act$.
The next internal state and action pair $(g_l,\alpha)$ is chosen by independent sampling of 
$\omega(g_k,o)$. By abuse of notation, $\omega(g_l,\alpha|g_k,o)$ will denote the probability of
transitioning to internal stochastic finite state controller state $g_l$ and taking action $\alpha$, when the current internal
state is $g_k$ and observation $o$ is received. $\kappa:\Delta({\mathcal{S}}) \to \Delta(G)$ chooses the starting internal FSC state $g_0$, by independent
sampling of $\kappa(\kappa_0)$, given initial distribution $\kappa_0$ of $\mathcal{PM}$, and
$\kappa(g|\kappa_0)$ will denote the probability of starting the FSC in internal state $g$ when
the initial POMDP distribution is $\kappa_0$.

\subsection{Coherent Risk Measures}

Consider a probability space $(\Omega, \mathcal{F}, \mathbb{P})$, a filteration $\mathcal{F}_0 \subset \cdots \mathcal{F}_N \subset \mathcal{F} $, and an adapted sequence of random variables~(stage-wise costs) $c_t,~t=0,\ldots, N$, where $N \in \mathbb{N}_{\ge 0} \cup \{\infty\}$.
%We posit that $\mathcal{F}_0=\{\Omega, \emptyset\}$, \textit{i.e.}, $Z_1$ is deterministic.
For $t=0,\ldots,N$, we further define the spaces $\mathcal{C}_t = \mathcal{L}_p(\Omega, \mathcal{F}_t, \mathbb{P})$, $p \in [1,\infty)$,  $\mathcal{C}_{t:N}=\mathcal{C}_t\times \cdots \times \mathcal{C}_N$ and $\mathcal{C}=\mathcal{C}_0\times \mathcal{C}_1 \times \cdots$. We  assume that the sequence $\boldsymbol{c} \in \mathcal{C}$ is almost surely bounded (with exceptions having probability zero), \textit{i.e.}, 
 $
\max_t \esssup~| c_t(\omega) | < \infty.
$

In order to describe how one can evaluate the risk of sub-sequence $c_t,\ldots, c_N$ from the perspective of stage $t$, we require the following definitions.

% Consider a probability space $(\Omega, \mathcal{F}, P)$, a filteration $\{ \mathcal{F}_t \}$ on  $(\Omega, \mathcal{F}, P)$, and an adapted sequence of random variables~(stage-wise costs) $c_t,~t=1,2,\ldots$.
% %We posit that $\mathcal{F}_0=\{\Omega, \emptyset\}$, \textit{i.e.}, $Z_1$ is deterministic.
% We further define the spaces $\mathcal{C}_t = \mathcal{L}_p(\Omega, \mathcal{F}_t, P)$, $p \in [0,\infty)$, $t=1,2,\ldots$ and let $\mathcal{C}=\mathcal{C}_1\times \mathcal{C}_2 \times \cdots$. We further assume that the sequence $Z \in \mathcal{C}$ is almost surely bounded, \textit{i.e.}, 
% $$
% \max_t \mathrm{essup}~| c_t(\omega) | < \infty.
% $$
\vspace{0.2cm}
\begin{defn}[Conditional Risk Measure]{
A mapping $\rho_{t:N}: \mathcal{C}_{t:N} \to \mathcal{C}_{t}$, where $0\le t\le N$, is called a \emph{conditional risk measure}, if it has the following monoticity property:
\begin{equation*}
    \rho_{t:N}(\boldsymbol{c}) \le   \rho_{t:N}(\boldsymbol{c}'), \quad \forall \boldsymbol{c}, \forall \boldsymbol{c}' \in \mathcal{C}_{t:N}~\text{such that}~\boldsymbol{c} \preceq \boldsymbol{c}'.
\end{equation*}
}
\end{defn}
\vspace{0.2cm}
\begin{defn}[Dynamic Risk Measure]
{A \emph{dynamic risk measure} is a sequence of conditional risk measures $\rho_{t:N}:\mathcal{C}_{t:N}\to \mathcal{C}_{t}$, $t=0,\ldots,N$.}
\end{defn}
\vspace{0.2cm}
One fundamental property of dynamic risk measures is their consistency over time~\cite[Definition 3]{ruszczynski2010risk}. That is, if $c$ will be as good as $c'$ from the perspective of some future time $\theta$, and they are identical between time $\tau$ and $\theta$, then $c$ should not be worse than $c'$ from the perspective at time $\tau$.
% \vspace{0.2cm}
% \begin{defn}[Time-Consistent Risk Measure]{
% A dynamic risk measure $\left\{ \rho_{t:N}  \right\}_{t=0}^T$ is called \emph{time-consistent} if for all $0\le t \le \tau < \theta  \le T$ and all sequences $Z,W \in \mathcal{C}_{t:N}$ the conditions
% \begin{multline*}
% c_t =c'_t,~t=\tau,\ldots,\theta-1,~~\text{and}~\\ \rho_{\theta,T}(Z_\theta,\ldots,c_t) \le  \rho_{\theta,T}(W_\theta,\ldots,c'_t)
% \end{multline*}
% imply
% \begin{equation}
% \rho_{\tau,N}(c_\tau,\ldots,c_t) \le  \rho_{\tau,N}(c'_\tau,\ldots,c'_t).
% \end{equation}
% }
% \end{defn}
% \vspace{0.2cm}
% If a risk measure is time-consistent, we can define the one-step conditional risk measure $\rho_t:\mathcal{C}_{t+1}\to \mathcal{C}_t$, $t=0,\ldots,N-1$ as follows:
% \begin{equation}
%     \rho_t(c_{t+1}) = \rho_{t,t+1}(0,c_{t+1}),
% \end{equation}
% and for all $t=1,\ldots,N$, we obtain:
% \begin{multline}
%     \label{eq:dynriskmeasure}
%     \rho_{t,N}(c_t,\ldots,c_N)= \rho_t \big(c_t + \rho_{t+1} ( c_{t+1}+\rho_{t+2}(c_{t+2}+\cdots\\
%     +\rho_{N-1}\left(c_{N-1}+\rho_{N}(c_N) \right) \cdots )) \big).
% \end{multline}
% Note that the time-consistent risk measure is completely defined by one-step conditional risk measures $\rho_t$, $t=0,\ldots,N-1$ and, in particular, for $t=0$, \eqref{eq:dynriskmeasure} defines a risk measure of the entire sequence $\boldsymbol{c} \in \mathcal{C}_{0:N}$.

In this paper, we focus on time consistent, coherent risk measures, which satisfy four nice mathematical properties, as defined below~\cite[p. 298]{shapiro2014lectures}. 

\vspace{0.2cm}
\begin{defn}[Coherent Risk Measure]\label{defi:coherent}{
We call the one-step conditional risk measures $\rho_t: \mathcal{C}_{t+1}\to \mathcal{C}_t$, $t=1,\ldots,N-1$ a \emph{coherent risk measure} if it satisfies the following conditions
\begin{itemize}
    \item \textbf{Convexity:} $\rho_t(\lambda c + (1-\lambda)c') \le \lambda \rho_t(c)+(1-\lambda)\rho_t(c')$, for all $\lambda \in (0,1)$ and all $c,c' \in \mathcal{C}_{t+1}$;
    \item \textbf{Monotonicity:} If $c\le c'$ then $\rho_t(c) \le \rho_t(c')$ for all $c,c' \in \mathcal{C}_{t+1}$;
    \item \textbf{Translational Invariance:} $\rho_t(c+c')=c+\rho_t(c')$ for all $c \in \mathcal{C}_t$ and $c' \in \mathcal{C}_{t+1}$;
    \item \textbf{Positive Homogeneity:} $\rho_t(\beta c)= \beta \rho_t(c)$ for all $c \in \mathcal{C}_{t+1}$ and $\beta \ge 0$.
\end{itemize}
}
\end{defn}
\vspace{0.2cm}

We are interested in the discounted infinite horizon problems. Let $\gamma \in (0,1)$ be a given discount factor. For $t=0,1,\ldots$, we define the functional
\begin{multline}
    \rho^\gamma_{0,t}(c_0,\ldots,c_t) = \rho_{0,t}(c_0,\gamma c_1,\ldots, \gamma^{t}c_t) \nonumber \\
                                      = \rho_0 \bigg(c_0 + \rho_{1} \big( \gamma c_{1}+\rho_{2}(\gamma^2c_{2}+\cdots \nonumber\\
                                      ~~+\rho_{t-1}\left(\gamma^{t-1}c_{t-1}+\rho_{t}(\gamma^{t}c_t) \right) \cdots )\big) \bigg).
\end{multline}
 Finally, we have total discounted risk functional $\rho^{\gamma}:\mathcal{C}\to \mathbb{R}$ defined as \begin{equation}\label{eq:totaldiscrisk} \rho^{\gamma}(\boldsymbol{c}) = \lim_{t \to \infty} \rho^\gamma_{0,t}(c_0,\ldots,c_t).\end{equation} From~\cite[Theorem 3]{ruszczynski2010risk}, we have that $\rho^{\gamma}$ is convex, monotone, and positive homogeneous.

\subsection{Examples of Coherent Risk Measures} \label{sec:riskexamples}

Next, we briefly review three examples of coherent risk measures that will be used in this paper.

\textbf{Total Conditional Expectation:} The simplest risk measure is the total conditional expectation given by
\begin{equation}
    \rho_t(c_{t+1}) =  \mathbb{E}\left[  c_{t+1} \mid \mathcal{F}_{t}    \right].
\end{equation} 
It is easy to see that total conditional expectation satisfies the properties of a coherent risk measure as outlined in Definition~\ref{defi:coherent}. Unfortunately, total conditional expectation is agnostic to realization fluctuations of the random variable $c$ and is only concerned with the mean value of $c$ at large number of realizations. Thus, it is a risk-neutral measure of performance.

\textbf{Conditional Value-at-Risk:} Let $c \in \mathcal{C}$ be a random variable. For a given confidence level $\varepsilon \in (0,1)$, value-at-risk ($\mathrm{VaR}_{\varepsilon}$) denotes the $(1-\varepsilon)$-quantile value of the random  variable $c \in \mathcal{C}$. Unfortunately, working with VaR  for non-normal random variables is numerically unstable and optimizing models involving  VaR is intractable in
high dimensions~\cite{rockafellar2000optimization}. 

In contrast, CVaR overcomes the shortcomings of VaR. CVaR with confidence level $\varepsilon \in (0,1)$ denoted $\mathrm{CVaR}_{\varepsilon}$ measures the expected loss in the $(1-\varepsilon)$-tail given that the particular threshold $\mathrm{VaR}_{\varepsilon}$ has been crossed, i.e., $\mathrm{CVaR}_{\varepsilon} (c) =  \mathbb{E}\left[ c \mid c \ge \mathrm{VaR}_{\varepsilon}(c)  \right]$. An optimization formulation for CVaR was proposed in~\cite{rockafellar2000optimization}. That is, $\mathrm{CVaR}_{\varepsilon}$ is given by 
\begin{multline}
   \rho_t(c_{t+1}) =  \mathrm{CVaR}_{\varepsilon}(c_{t+1}) \\:=\inf_{\zeta \in \mathbb{R}} \left(\zeta+\frac{1}{\varepsilon}\mathbb{E}\left[ (c_{t+1}-\zeta)_{+} \mid \mathcal{F}_t\right] \right), \label{eq:cvardual}
\end{multline}
 where $(\cdot)_{+}=\max\{\cdot, 0\}$. A value of $\varepsilon \to 1$ corresponds to a risk-neutral case, i.e.,  $\mathrm{CVaR_1}(c)=\mathbb{E}(c)$; whereas, a value of $\varepsilon \to 0$ is rather a risk-averse case, i.e., $\mathrm{CVaR_0}(c)=\mathrm{VaR}_0(c)= \essinf(c)$~\cite{rockafellar2002conditional}. Figure~\ref{fig:varvscvar} illustrates these notions for an example $c$ variable with distribution $p(c)$.
 
 \begin{figure} 
\centering{
\includegraphics[width=8.5cm]{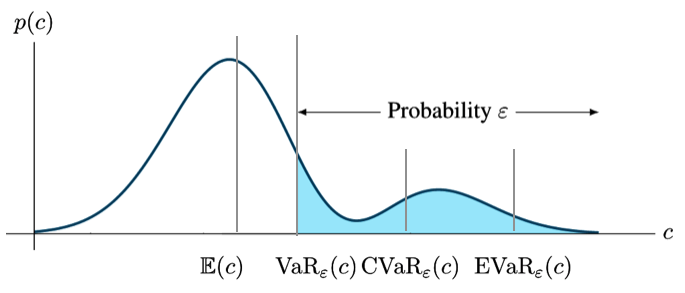}
}
\caption{Comparison of the mean, VaR, and CVaR for a given confidence $\varepsilon \in (0,1)$. The axes denote the values of the stochastic variable $c$ and its probability density function $p(c)$. The shaded area denotes the $\%\varepsilon$ of the area under $p(c)$. The expected cost $\mathbb{E}(c)$ is much smaller than the worst case cost.  VaR gives the value of $c$ at the $(1-\varepsilon)$-tail of the distribution. But, it ignores the values of $c$ with probability below $\varepsilon$.  CVaR is the average of the values of VaR with probability less than $\varepsilon$ (average of the worst-case values of $c$ in the $(1-\varepsilon)$ tail of the distribution).} \label{fig:varvscvar}
\end{figure}

\textbf{Entropic Value-at-Risk:} Unfortunately, CVaR ignores the losses below the VaR threshold. EVaR is the tightest upper bound in the sense of Chernoff inequality for VaR and CVaR and its dual representation is associated with the relative entropy. In fact, it was shown in~\cite{ahmadi2017analytical} that $\mathrm{EVaR}_\varepsilon$ and $\mathrm{CVaR}_\varepsilon$ are equal only if there are no losses ($c\to -\infty$) below the $\mathrm{VaR}_\varepsilon$ threshold. In addition, EVaR is a strictly monotone risk measure; whereas, CVaR is only monotone~\cite{ahmadi2019portfolio}. $\mathrm{EVaR}_\varepsilon$  is given by
\begin{equation}
    \rho_t(c_{t+1}) = \inf_{\zeta >0} \left(  {\log \left(\frac{\mathbb{E}[e^{\zeta c_{t+1}} \mid \mathcal{F}_{t}]}{\varepsilon}\right)/ \zeta}        \right).            
\end{equation}
Similar to $\mathrm{CVaR}_\varepsilon$, for $\mathrm{EVaR}_\varepsilon$, $\varepsilon \to 1$ corresponds to a risk-neutral case; whereas, $\varepsilon\to 0$ corresponds to a risk-averse case. In fact, it was demonstrated in~\cite[Proposition 3.2]{ahmadi2012entropic} that $\lim_{\varepsilon\to 0} \mathrm{EVaR}_{\varepsilon}(c) = \essinf(c)$. 

\section{Problem Formulation}

In the past two decades, coherent risk and dynamic risk measures have been developed and used in microeconomics and mathematical finance fields~\cite{vose2008risk}. Generally speaking, risk-averse decision making is concerned with the behavior of agents, e.g. consumers and investors, who, when exposed to uncertainty, attempt to lower that uncertainty. The agents may avoid situations with unknown payoffs, in favor of situations with payoffs that are more predictable.

The core idea in risk-averse planning is to replace the conventional risk-neutral conditional expectation of the cumulative cost objectives with the more general coherent risk measures. In path planning scenarios, in particular, we will show in our numerical experiments that considering coherent risk measures will lead to significantly more robustness to environment uncertainty and collisions leading to mission failures.

 In addition to total cost risk-aversity, an agent is often subject to constraints, e.g. fuel, communication, or energy budgets~\cite{7452536}. These constraints can also represent mission objectives, e.g. explore an area or reach a goal.  

% Let $d^i: \mathcal{S} \times \mathcal{A} \to \mathbb{R}_{\ge 0}$ with $i=1,2,\ldots,n_c$, be a set of constraint cost functions. 

Consider a stationary controlled Markov process $\{q_t\}$, $t=0,1,\ldots$ (an MDP or a POMDP) with initial probability distribution $\kappa_0$, wherein policies, transition probabilities, and cost functions do not depend explicitly on time. Each policy $\pi = \{\pi_t\}_{t=0}^\infty$ leads to  cost sequences $\boldsymbol{c}_t=c(q_t,\alpha_t)$, $t=0,1,\ldots$ and $\boldsymbol{d}_t^i=d^i(q_t,\alpha_t)$, $t=0,1,\ldots$, $i=1,2,\ldots,n_c$. We define the dynamic risk of evaluating the $\gamma$-discounted cost of a policy $\pi$ as
\begin{equation}\label{eq:objrisk}
    J_{\gamma}(\kappa_0,\pi) = \rho^{\gamma} \big( c(q_0,\alpha_0),c(q_1,\alpha_1),\ldots \big),
\end{equation}
 and the $\gamma$-discounted dynamic risk constraints of executing policy $\pi$ as
\begin{multline}\label{eq:constraint}
 D_{\gamma}^i(\kappa_0,\pi)=   \rho^{\gamma}\left(  d^i(q_0,\alpha_0),d^i(q_1,\alpha_1),\ldots \right) \le \beta^i, \\ i=1,2,\ldots,n_c,
\end{multline}
where $\rho^{\gamma}$ is defined in equation~\eqref{eq:totaldiscrisk}, $q_0 \sim \kappa_0$, and $\beta^i>0$, $i=1,2,\ldots,n_c$, are given constants. We assume that $c(\cdot,\cdot)$ and $d^i(\cdot,\cdot)$, $i=1,2,\ldots,n_c$, are non-negative and upper-bounded. For a discount factor $\gamma \in (0,1)$, an initial condition $\kappa_0$, and a policy $\pi$, we infer from~\cite[Theorem 3]{ruszczynski2010risk} that both $J_{\gamma}(\kappa_0,\pi)$ and $D_{\gamma}^i(\kappa_0,\pi)$ are well-defined (bounded), if $c$ and $d$ are bounded. 

In this work, we are interested in addressing the following problem:
\vspace{0.2cm}
\begin{problem}\textit{
For a controlled Markov decision process (an MDP or a POMDP), a discount factor $\gamma \in (0,1)$, and a total risk  functional $J_{\gamma}(\kappa_0,\pi)$ as in equation~\eqref{eq:objrisk} and total cost constraints~\eqref{eq:constraint}, where $\{\rho_t\}_{t=0}^\infty$ are coherent risk measures, compute 
\begin{align}
\pi^* \in &~\argmin_{\pi} ~~J_{\gamma}(\kappa_0,\pi) \nonumber \\ & \text{subject to} \quad \boldsymbol{D}_{\gamma}(\kappa_0,\pi) \preceq \boldsymbol{\beta}.
\end{align}
}
\end{problem}
\vspace{0.2cm}
We call a controlled Markov process with the ``nested'' objective~\eqref{eq:objrisk} and constraints~\eqref{eq:constraint} a  \emph{constrained risk-averse Markov process}. 

For MDPs, \cite{chow2015risk,osogami2012robustness} show that such coherent risk measure objectives can account for  modeling errors and parametric uncertainties. We can also interpret Problem 1 as designing policies that minimize the accrued costs in a risk-averse sense and at the same time ensuring that the system constraints, \textit{e.g.}, fuel constraints, are not violated even in the  rare but costly scenarios. 

Note that in Problem 1 both the objective function and the constraints are in general non-differentiable and non-convex in policy $\pi$ (with the exception of total expected cost as the coherent risk measure $\rho^\gamma$~\cite{altman1999constrained}). Therefore, finding optimal policies in general may be hopeless. Instead, we find sub-optimal polices by taking advantage of a Lagrangian formulation and then using an optimization form of Bellman's equations.

% Note that in the case where $\rho^\gamma$ is the total expected cost, we have a constrained Markov decision process~\cite{altman1999constrained}. Indeed, Problem 1 is a generalization of constrained Markov decision processes to coherent risk measures. Indeed, we show in Section~\ref{} that our approach converges to those of constrained Markov decision processes when $\rho^\gamma$ is the total expected cost.

% Note that since the constraints~\eqref{eq:constraint} are affine in the policy variable~$\pi$, there always exist a solution that solves the constrained optimization problem. However, the problem of minimizing a total cost function subject to constraints on total coherent risk is not well-posed in general, since total coherent risk is not necessarily a convex function of $\pi$ (\cite{chow2014algorithms} attempts to minimze total cost subject to CVaR constraints assuming an optimal policy exists).

Next, we show that the constrained risk-averse problem is equivalent to a non-constrained inf-sup risk-averse problem thanks to the Lagrangian method. 

\begin{proposition}
Let $J_\gamma(\kappa_0)$ be the  value of Problem 1 for a given initial distribution $\kappa_0$ and discount factor $\gamma$. Then, (i) the value function satisfies 
\begin{equation}\label{eq:Dds}
    J_\gamma(\kappa_0) = \inf_{\pi}\sup_{\boldsymbol{\lambda} \succeq \boldsymbol{0}} L_{\gamma}(\pi,{\boldsymbol{\lambda}}),
\end{equation}
where 
\begin{align}\label{eq:lagrangian}
    L_{\gamma}(\pi,\boldsymbol{\lambda}) = J_{\gamma}(\kappa_0,\pi)+ \langle \boldsymbol{\lambda},\left(\boldsymbol{D}_{\gamma}(\kappa_0,\pi)-\boldsymbol{\beta}\right) \rangle,
\end{align}
is the Lagrangian function.\\
(ii) Furthermore, a policy $\pi^*$ is optimal for~Problem 1, if and only if $J_\gamma(\kappa_0)=\sup_{\boldsymbol{\lambda} \succeq \boldsymbol{0}}~L_{\gamma}(\pi^*,\boldsymbol{\lambda})$.
\end{proposition}
%\vspace{0.2cm}
\begin{proof}
(i) If for some $\pi$ Problem 1 is not feasible, then $\sup_{\boldsymbol{\lambda} \succeq \boldsymbol{0}} L_{\gamma}(\pi,{\lambda})=\infty$. In fact, if the $i$th constraint is not satisfied, i.e., $D_\gamma^i > \beta^i$, we can achieve the latter supremum by choosing $\lambda_i\to \infty$, while keeping the rest of $\lambda^i$s constant or zero. If Problem 1 is feasible for some $\pi$, then the supremum is achieved by setting $\boldsymbol{\lambda}=\boldsymbol{0}$. Hence, $L_\gamma(\lambda,\pi)=J_\gamma(\kappa_0,\pi)$ and 
$$
\inf_{\pi} \sup_{\boldsymbol{\lambda} \succeq \boldsymbol{0}}~L_{\gamma}(\pi,{\lambda}) = \inf_{\pi : \boldsymbol{D}_\gamma(\kappa_0,\pi) \le \boldsymbol{\beta}}~~J_\gamma(\kappa_0,\pi),
$$
which implies~(i). \\
(ii) If $\pi$ is optimal, then, from~\eqref{eq:Dds}, we have 
$$
J_\gamma(\kappa_0) = \sup_{\boldsymbol{\lambda} \succeq \boldsymbol{0}} L_{\gamma}(\pi^*,{\lambda}).
$$
Conversely, if $J_\gamma(\kappa_0) = \sup_{\boldsymbol{\lambda} \succeq \boldsymbol{0}} L_{\gamma}(\pi',{\lambda})$ for some $\pi'$, then from~\eqref{eq:Dds}, we have $\inf_\pi \sup_{\boldsymbol{\lambda} \succeq \boldsymbol{0}} L_\gamma(\pi,\lambda) = \sup_{\boldsymbol{\lambda} \succeq \boldsymbol{0}} L_{\gamma}(\pi',{\lambda})$. Hence, $\pi'$ is the optimal policy. 
\end{proof}
\vspace{0.2cm}

% With Lagrangian formulation in Proposition 1, we can assert the following about Problem 1. 

% \begin{proposition}
% Problem 1 is well-posed. 
% \end{proposition}
% \begin{proof}
% It suffices to show that $J_\gamma (\kapp_0)$ is bounded for 
% \end{proof}

\section{Constrained Risk-Averse MDPs} \label{sec:mdps}

At any time $t$, the value of $\rho_t$ is $\mathcal{F}_t$-measurable and is allowed to depend on the entire history of the process $\{s_0,s_1,\ldots\}$ and we cannot expect to obtain a Markov optimal policy~\cite{ott2010markov,bauerle2011markov}. In order to obtain Markov policies, we need the following property~\cite{ruszczynski2010risk}. 
\vspace{0.2cm}

\begin{defn}[Markov Risk Measure]\label{assum1}\textit{
Let $m,n \in [1,\infty)$ such that $1/m+1/n=1$ and 
$
\mathcal{P} = \big\{p \in \mathcal{L}_n(\mathcal{S}, 2^\mathcal{S}, \mathbb{P}) \mid \sum_{s' \in \mathcal{S}} p(s') \mathbb{P}(s')=1,~p\ge 0 \big\}.
$ 
A one-step conditional risk measure $\rho_t:\mathcal{C}_{t+1}\to \mathcal{C}_t$ is a Markov risk measure with respect to the controlled Markov process $\{s_t\}$, $t=0,1,\ldots$, if there exist a risk transition mapping $\sigma_t: \mathcal{L}_m(\mathcal{S}, 2^\mathcal{S}, \mathbb{P}) \times \mathcal{S} \times \mathcal{P} \to \mathbb{R}$ such that for all $v \in \mathcal{L}_m(\mathcal{S}, 2^\mathcal{S}, \mathbb{P})$ and $\alpha_t \in \pi(s_t)$, we have
\begin{equation}
    \rho_t(v(s_{t+1})) = \sigma_t(v(s_{t+1}),s_t,p(s_{t+1}|s_t,\alpha_t)),
\end{equation}
where $p:\mathcal{S}\times Act \to \mathcal{P}$ is called the controlled kernel.
}
\end{defn}
\vspace{0.2cm}

In fact, if $\rho_t$ is a coherent risk measure, $\sigma_t$ also satisfies the properties of a coherent risk measure (Definition 3). In this paper, since we are concerned with MDPs, the controlled kernel is simply the transition function $T$.
\vspace{0.2cm}
\begin{assum}\label{assum1}\textit{
The one-step coherent risk measure $\rho_t$ is a Markov risk measure.
}
\end{assum}
\vspace{0.2cm}
% Note that the Markov risk transition mapping depends on the function $\phi$, the states $s$, and probability vector $p(s,\alpha)$. The dot in $\phi(s_t,\alpha_t, \cdot)$ on the right hand side of~\eqref{eq:markovtrans} represents a dummy variable that is integrated/summed out with respect to the $s_t$-th row of the transition probability matrix $p(s_t,\alpha_t)$. 

The simplest case of the risk
transition mapping is in the
conditional expectation case
$\rho_t(v(s_{t+1})) =
\mathbb{E}\{v(s_{t+1}) \mid
s_t,\alpha_t\}$, \textit{i.e.}, 
\begin{multline}\label{eq:fdffd}
\sigma\left\{v(s_{t+1}),s_t,p(s_{t+1}|s_t,\alpha_t) \right\} =
\mathbb{E}\{ v(s_{t+1})\mid s_t,\alpha_t\} \\ =
\sum_{s_{t+1} \in \mathcal{S}} v(s_{t+1}) T(s_{t+1}\mid s_t,\alpha_t).
\end{multline}
Note that in the total discounted expectation case $\sigma$ is a linear function in $v$ rather than a convex function, which is the case for a general coherent risk measures. For example, for the CVaR risk measure, the Markov risk transition mapping is given by
 \begin{multline*}
    \sigma \{ v(s_{t+1}),s_t,p(s_{t+1}|s_t,\alpha_t)  \} \\= \inf_{\zeta \in \mathbb{R}} \left\{ \zeta + \frac{1}{\varepsilon} \sum_{s_{t+1} \in \mathcal{S}}\left(v(s_{t+1})-\zeta\right)_{+} T(s_{t+1}\mid s_t,\alpha_t)        \right\},
 \end{multline*}
where $(\cdot)_{+}=\max\{\cdot, 0\}$ is a convex function in $v$. 

If $\sigma$ is a coherent, Markov risk measure, then the Markov policies are sufficient to ensure optimality~\cite{ruszczynski2010risk}.

In the next result, we show that we can find a lower bound to the solution to Problem 1 via solving an optimization problem.

\vspace{0.2cm}
\begin{theorem}\textit{
Consider an MDP~$\mathcal{M}$   with the nested risk objective~\eqref{eq:objrisk},  constraints~\eqref{eq:constraint}, and discount factor $\gamma \in (0,1)$. Let Assumption~\ref{assum1} hold and $\rho_t,~t=0,1,\ldots$ be  coherent risk measures as described in Definition~\ref{defi:coherent}. Then, the solution $(\boldsymbol{V}^*_\gamma,\boldsymbol{\lambda}^*)$ to the following optimization problem (Bellman's equation)
\begin{align}\label{eq:valueiteration}
  & \sup_{\boldsymbol{V}_\gamma,\boldsymbol{\lambda} \succeq \boldsymbol{0}}~~\langle \boldsymbol{\kappa_0},\boldsymbol{V}_\gamma\rangle - \langle \boldsymbol{\lambda},\boldsymbol{\beta} \rangle \nonumber  \\
        &\text{subject to} \nonumber  \\
        &V_\gamma(s) \le c(s,\alpha) + \langle \boldsymbol{\lambda}, \boldsymbol{d}(s,\alpha)\rangle \nonumber \\ &\quad \quad \quad         +\gamma \sigma\{ {V}_\gamma(s'),s,p(s'|s,\alpha) \},~\forall s \in \mathcal{S},~\forall \alpha \in {Act,}
% V(s) =& \min_{\alpha \in Act} \sup_{\boldsymbol{\lambda} \succeq \boldsymbol{0}} \Big( c(s,\alpha) + \lambda {d}(s,a)
%      \\&\quad \quad \quad \quad+\gamma \sigma\{ V^*(s'),s,p(s'|s,\alpha) \}\Big),~~\forall s \in \mathcal{S},
    %   & \sup_{V,\boldsymbol{\lambda} \succeq \boldsymbol{0}}~~\langle \kappa_0,V \rangle - \langle \lambda,\beta \rangle \nonumber  \\
    %     &\text{subject to} \nonumber  \\
    %     &V(b(s)) \le c(b(s),\alpha) + \langle \boldsymbol{\lambda}, \boldsymbol{d}(b(s),a)\rangle \nonumber 
    %     \\&\quad \quad \quad \quad+\gamma \sigma\{ V^*(b'(s)),b(s),p(b'(s)|b(s),\alpha) \},
\end{align}
satisfies
\begin{equation} \label{eq:lowerboundrisk}
    J_\gamma(\kappa_0) \ge \langle \boldsymbol{\kappa_0},\boldsymbol{V}^*_\gamma \rangle-\langle\boldsymbol{\lambda}^*,\boldsymbol{\beta}\rangle.
\end{equation}
}
\end{theorem}
\vspace{0.2cm}
\begin{proof}
From Proposition 1, we have know that~\eqref{eq:Dds} holds. Hence, we have
\begin{align} \label{eq:sddsfdsfdsfsdfaaa}
    J_\gamma(\kappa_0) &= \inf_\pi \sup_{\boldsymbol{\lambda} \succeq \boldsymbol{0}} \left( J_{\gamma}(\kappa_0,\pi)+ \langle \lambda,\left(\boldsymbol{D}_{\gamma}(\kappa_0,\pi)-\boldsymbol{\beta}\right) \rangle \right) \nonumber \\
    & = \inf_\pi \sup_{\boldsymbol{\lambda} \succeq \boldsymbol{0}} \left( J_{\gamma}(\kappa_0,\pi)+\langle \lambda,\boldsymbol{D}_{\gamma}(\kappa_0,\pi)\rangle  -\langle \lambda,\beta\rangle \right) \nonumber \\ &=
    \inf_\pi \sup_{\boldsymbol{\lambda} \succeq \boldsymbol{0}} \left( \rho^\gamma(c)+ \langle \lambda , \rho^\gamma(d) \rangle-\langle \lambda, \beta \rangle \right) \nonumber \\
  & = \inf_\pi \sup_{\boldsymbol{\lambda} \succeq \boldsymbol{0}} \left( \rho^\gamma(c)+  \rho^\gamma(\langle\lambda ,d \rangle)-\langle\lambda, \beta \rangle \right) \nonumber \\
  & \ge \inf_\pi \sup_{\boldsymbol{\lambda} \succeq \boldsymbol{0}} \left( \rho^\gamma(c+\langle \lambda, d\rangle)-\langle \lambda, \beta \rangle \right), \nonumber \\ 
& \ge \sup_{\boldsymbol{\lambda} \succeq \boldsymbol{0}} \inf_\pi \left( \rho^\gamma(c+\langle \lambda, d\rangle)-\langle \lambda, \beta \rangle \right)
\end{align}
wherein the fourth, fifth, and sixth inequalities above we used the positive homogeneity property of $\rho^\gamma$, sub-additivity property of $\rho^\gamma$, and the minimax inequality respectively. Since $\langle \lambda, \beta \rangle$ does not depend on $\pi$, to find the solution the infimum it suffices to find the solution to  
$$
\inf_\pi \rho^\gamma(\tilde{c}),
$$
where $\tilde{c}=c+\lambda'd$. The value to the above optimization can be obtained by solving the following Bellman equation~\cite[Theorem 4]{ruszczynski2010risk}
$$
V_\gamma(s)=\inf_{\alpha \in Act} \Big( \tilde{c}(s,\alpha)+\gamma \sigma\{ V_\gamma(s'),s,p(s'|s,\alpha) \} \Big).
$$
Next, we show that the solution to the above Bellman equation can be alternatively obtained by solving the convex optimization
\begin{align}\label{eq:valueiteration2}
  & \sup_{V_\gamma}~~\langle \kappa_0,V_\gamma\rangle  \nonumber  \\
        &\text{subject to} \nonumber  \\
        &V_\gamma(s) \le \tilde{c}(s,\alpha) + \gamma \sigma\{ V_\gamma(s'),s,p(s'|s,\alpha) \},~\forall s,\alpha.
\end{align}
Define $$\mathfrak{D}_\pi v : = \tilde{c}(s,\pi(s)) + \gamma \sigma\{ v(s'),s,p(s'|s,\pi(s)) \}, \quad \forall s \in \mathcal{S},$$ and $\mathfrak{D} v : = \min_{\alpha \in Act} \left( \tilde{c}(s,\alpha) + \gamma \sigma\{ v(s'),s,p(s'|s,\alpha) \} \right)$ for all $s \in \mathcal{S}$. From \cite[Lemma 1]{ruszczynski2010risk}, we infer that $\mathfrak{D}_\pi$ and $\mathfrak{D}$ are non-decreasing; i.e., for $v\le w$, we have $\mathfrak{D}_\pi v \le \mathfrak{D}_\pi w$ and $\mathfrak{D} v \le \mathfrak{D} w$. Therefore, if $V_\gamma \le \mathfrak{D}_\pi V_\gamma$, then $\mathfrak{D}_\pi V_\gamma \le \mathfrak{D}_\pi(\mathfrak{D}_\pi V_\gamma)$. By repeated application of $\mathfrak{D}_\pi$, we obtain
$$
V_\gamma \le \mathfrak{D}_\pi V_\gamma \le \mathfrak{D}_\pi^2 V_\gamma \le \mathfrak{D}_\pi^\infty V_\gamma=V^*_\gamma.
$$
Any feasible solution to~\eqref{eq:valueiteration2} must satisfy $V_\gamma \ge \mathfrak{D}_\pi V_\gamma$ and hence must satisfy $V_\gamma \ge  V^*_\gamma$. Thus, given that all entries of $\kappa_0$ are positive, $V^*_\gamma$ is the optimal solution to~\eqref{eq:valueiteration2}. Substituting~\eqref{eq:valueiteration2}  back in the last inequality in~\eqref{eq:sddsfdsfdsfsdfaaa} yields the result.  
\end{proof}
\vspace{.2cm}

Once the values of $\boldsymbol{\lambda}^*$ and $\boldsymbol{V}^*_\gamma$ are found by solving optimization problem~\eqref{eq:valueiteration}, we can find the policy as
\begin{align}
    \pi^*(s) \in &~\argmin_{\alpha \in Act}~\Big(  c(s,\alpha) + \langle \boldsymbol{\lambda}^*, \boldsymbol{d}(s,\alpha)\rangle        \nonumber \\ & \quad \quad \quad \quad +\gamma \sigma\{ V^*_\gamma(s'),s,p(s'|s,\alpha) \}   \Big).
\end{align}

% Note that $\pi^*$ is a deterministic, stationary policy. Such policies are desirable in practical robotics applications, since they do not require much memory and . Given an uncertain environment, $\pi^*$ can be designed offline and used for path planning. 

One interesting observation is that if the coherent risk measure $\rho^t$ is the total discounted expectation, then  Theorem 1 is consistent with the classical result by~\cite{altman1999constrained} on constrained MDPs. 
\vspace{.2cm}
\begin{corollary}
Let the assumptions of Theorem 1 hold and let $\rho_t(\cdot) = \mathbb{E}(\cdot |s_t,\alpha_t)$, $t=1,2,\ldots$. Then the solution $(\boldsymbol{V}^*_\gamma,\boldsymbol{\lambda}^*)$ to  optimization~\eqref{eq:valueiteration} satisfies
$$
    J_\gamma(\kappa_0) = \langle \boldsymbol{\kappa_0},\boldsymbol{V}^*_\gamma \rangle-\langle\boldsymbol{\lambda}^*,\boldsymbol{\beta}\rangle.
    $$
    Furthermore, with $\rho_t(\cdot) = \mathbb{E}(\cdot |s_t,\alpha_t)$, $t=1,2,\ldots$, optimization~\eqref{eq:valueiteration} becomes a linear program.
\end{corollary}
\vspace{0.2cm}
\begin{proof}
From the derivation in~\eqref{eq:sddsfdsfdsfsdfaaa}, we observe the two inequalities are from the application of (a) the sub-additivity property of $\rho^\gamma$ and (b) the max-min inequality. Next, we show that in the case of total expectation both of these properties lead to an equality.\\
(a) Sub-additivity property of $\rho^\gamma$: for total expectation, we have  $$\sum_t \mathbb{E}_{\kappa_0}^\pi \gamma^t c_t+ \sum_t \mathbb{E}_{\kappa_0}^\pi \gamma^t \langle \boldsymbol{\lambda}, \boldsymbol{d}_t \rangle=\sum_t \mathbb{E}_{\kappa_0}^\pi \gamma^t (c_t+\langle \boldsymbol{\lambda}, \boldsymbol{d}_t \rangle).$$ Thus, equality holds. \\
(b) Max-min inequality: in the $\rho^\gamma_{\kappa_0}(\cdot) = \sum_t \mathbb{E}_{\kappa_0}^\pi \gamma^t (\cdot)$ case, both the objective function and the constraints are linear in the decision variables $\pi$ and $\boldsymbol{\lambda}$. Therefore, the sixth line in~\eqref{eq:sddsfdsfdsfsdfaaa} reads as
\begin{align}
& \inf_\pi \sup_{\boldsymbol{\lambda} \succeq \boldsymbol{0}} \left( \rho^\gamma(\boldsymbol{c}+\langle \boldsymbol{\lambda}, \boldsymbol{d}\rangle)-\langle \boldsymbol{\lambda}, \boldsymbol{\beta} \rangle \right) \nonumber \\ 
& = \inf_\pi \sup_{\boldsymbol{\lambda} \succeq \boldsymbol{0}} \left(\sum_{t} \mathbb{E}_{\kappa_0}^\pi \gamma^t (c_t + \langle \boldsymbol{\lambda},\boldsymbol{d}_t\rangle) -\langle \boldsymbol{\lambda}, \boldsymbol{\beta} \rangle \right).
\end{align}
Since the expression inside parantheses above is convex in $\pi$ ($\mathbb{E}_{\kappa_0}^\pi$ is linear in the policy) and concave (linear) in $\boldsymbol{\lambda}$. From Minimax Theorem~\cite{du2013minimax}, we have that the following equality holds
\begin{align*}
&\inf_\pi \sup_{\boldsymbol{\lambda} \succeq \boldsymbol{0}} \left(\sum_{t} \mathbb{E}_{\kappa_0}^\pi \gamma^t (c_t + \langle \boldsymbol{\lambda},\boldsymbol{d}_t\rangle) -\langle \boldsymbol{\lambda}, \boldsymbol{\beta} \rangle \right) \nonumber \\
& = \sup_{\boldsymbol{\lambda} \succeq \boldsymbol{0}} \inf_\pi  \left(\sum_{t} \mathbb{E}_{\kappa_0}^\pi \gamma^t (c_t + \langle \boldsymbol{\lambda},\boldsymbol{d}_t\rangle) -\langle \boldsymbol{\lambda}, \boldsymbol{\beta} \rangle \right).
\end{align*}

Furthermore, from~\eqref{eq:fdffd}, we see that $\sigma$ is linear in $v$ for total expectation. Therefore, the constraint in~\eqref{eq:valueiteration} is  linear in $V_\gamma$ and $\lambda$. Since $\langle \boldsymbol{\kappa_0},\boldsymbol{V}_\gamma\rangle - \langle \boldsymbol{\lambda},\boldsymbol{\beta} \rangle$ is also linear  in $V_\gamma$s and $\lambda$s, optimization~\eqref{eq:valueiteration} becomes a linear program in the case of total expectation coherent risk measure.  
\end{proof}

In~\cite{ahmadi2021aaai}, we presented a method based on difference convex programs to solve~\eqref{eq:valueiteration}, when $\rho^\gamma$ is an arbitrary coherent risk measure and we described the specific structure of the optimization problem for conditional expectation, CVaR, and EVaR. In fact, it was shown that~\eqref{eq:valueiteration} can be written in a standard DCP format as
\begin{align}\label{eq:DCP}
      & \inf_{\boldsymbol{V}_\gamma,\boldsymbol{\lambda} \succeq \boldsymbol{0}}~~ f_0(\boldsymbol{\lambda})-g_0(\boldsymbol{V}_\gamma)  \nonumber  \\
        &\text{subject to} \nonumber  \\
        &f_1({V}_\gamma)-g_1(\boldsymbol{\lambda})-g_2({V}_\gamma) \le 0, ~~\forall s,\alpha.
\end{align}
Optimization problem~\eqref{eq:DCP} is a standard DCP~\cite{horst1999dc}. DCPs arise in  many applications, such as feature selection in  machine learning~\cite{le2008dc} and inverse covariance estimation in statistics~\cite{thai2014inverse}. Although DCPs can be solved globally~\cite{horst1999dc}, \textit{e.g.} using branch and bound algorithms~\cite{lawler1966branch}, a locally optimal solution can be obtained based on techniques of nonlinear optimization~\cite{Bertsekas99} more efficiently. In particular, in this work, we use a variant of the convex-concave procedure~\cite{lipp2016variations,shen2016disciplined}, wherein  the concave terms are replaced by a convex upper bound and solved. In fact, the disciplined convex-concave programming (DCCP)~\cite{shen2016disciplined} technique linearizes DCP problems into a (disciplined) convex program (carried out automatically via the DCCP Python package~\cite{shen2016disciplined}), which is then converted into an equivalent cone program by
replacing each function with its graph implementation. Then, the cone program can be solved readily by available convex programming solvers, such as CVXPY~\cite{diamond2016cvxpy}. 

We end this section by pointing out that solving~\eqref{eq:valueiteration} using the 
DCCP method, only finds the (local) saddle points to optimization problem ~\eqref{eq:valueiteration}. Nevertheless, every saddle point to~\eqref{eq:valueiteration} satisfies~\eqref{eq:lowerboundrisk} (from Theorem 1). In fact, every saddle point is a lower bound of the optimal value of Problem~1.

\section{Constrained Risk-Averse POMDPs} \label{sec:mdps}

Next, we show that, in the case of POMDPs, we can find a lower bound to the solution to Problem 1 via solving an infinite-dimensional optimization problem. Note that a POMDP is equivalent to a belief MDP $\{ b_t \}$, $t=1,2,\ldots$, where $b_t$ is defined in~\eqref{eq:beliefupdate}.

\vspace{0.2cm}
\begin{theorem}\textit{
Consider a POMDP~$\mathcal{PM}$  with the nested risk objective~\eqref{eq:objrisk}  and  constraint~\eqref{eq:constraint} with $\gamma \in (0,1)$. Let Assumption~\ref{assum1} hold, let $\rho_t,~t=0,1,\ldots$ be  coherent risk measures, and suppose $c(\cdot,\cdot)$ and $\{d^i(\cdot,\cdot)\}_{i=1}^{n_c}$ be non-negative and upper-bounded. Then, the solution $(\lambda^*,V^*_\gamma)$ to the following Bellman's equation
\begin{align}\label{eq:valueiterationpomdp}
  & \sup_{\boldsymbol{V}_\gamma,\boldsymbol{\lambda} \succeq 0}~~\langle \boldsymbol{b}_0,\boldsymbol{V}_\gamma\rangle - \langle \boldsymbol{\lambda},\boldsymbol{\beta} \rangle \nonumber  \\
        &\text{subject to} \nonumber  \\
        &V_\gamma(b) \le c(b,\alpha) + \langle \boldsymbol{\lambda}, \boldsymbol{d}(b,\alpha)\rangle \nonumber \\& \quad \quad \quad +\gamma \sigma\{ V_\gamma(b'),b,p(b'|b,\alpha) \},~~\forall b \in \Delta(\mathcal{S}),~\forall \alpha \in {Act,}
% V(s) =& \min_{\alpha \in Act} \sup_{\lambda \ge 0} \Big( c(s,\alpha) + \lambda {d}(s,a)
%      \\&\quad \quad \quad \quad+\gamma \mathrm{R}\{ V^*(s'),s,p(s'|s,\alpha) \}\Big),~~\forall s \in \mathcal{S},
    %   & \sup_{V,\lambda\ge 0}~~\langle \kappa_0,V \rangle - \langle \lambda,\beta \rangle \nonumber  \\
    %     &\text{subject to} \nonumber  \\
    %     &V(b(s)) \le c(b(s),\alpha) + \langle \boldsymbol{\lambda}, \boldsymbol{d}(b(s),a)\rangle \nonumber 
    %     \\&\quad \quad \quad \quad+\gamma \mathrm{R}\{ V^*(b'(s)),b(s),p(b'(s)|b(s),\alpha) \},
\end{align}
where $c(b,\alpha) = \sum_{s \in \mathcal{S}} c(s,\alpha)b(s)$ and $d(b,\alpha) = \sum_{s \in \mathcal{S}} d(s,\alpha)b(s)$ satisfies
\begin{equation} \label{eq:costpomdpbeliefs}
    J_\gamma(b_0) \ge \langle \boldsymbol{b}_0,\boldsymbol{V}^*_\gamma\rangle-\langle\boldsymbol{\lambda}^*,\boldsymbol{\beta}\rangle.
\end{equation}
}
\end{theorem}
\begin{proof}
Note that a POMDP can be represented as an MDP over the belief states~\eqref{eq:beliefupdate} with initial distribution~(1). Hence, a POMDP is a controlled Markov process with states $b \in \Delta(\mathcal{S})$, where the controlled belief transition probability is described as
\begin{multline*}
 p(b' \mid b,\alpha) = \sum_{o \in \mathcal{O}} p(b' \mid b, o, \alpha)~p(o \mid b, \alpha)  \\= \sum_{o \in \mathcal{O}}\delta \left(b' - \frac{O(o \mid s,\alpha)\sum_{s' \in \mathcal{S}} T(s \mid s',\alpha)b(s')}{\sum_{s \in \mathcal{S}} O(o \mid s,\alpha)\sum_{s' \in \mathcal{S}} T(s \mid s',\alpha)b(s')}\right) \\
 \times\sum_{s \in \mathcal{S}} O(o \mid s,\alpha)\sum_{s'' \in \mathcal{S}} T(s \mid s'',\alpha)b(s''),
\end{multline*}
with $$\delta(a) = \begin{cases} 1 & a=0, \\ 0 & \text{otherwise}. \end{cases}$$
The rest of the proof follows the same footsteps on Theorem~1 over the belief MDP with $p(b'|b,\alpha)$ as defined above.  
\end{proof}

Unfortunately, since $b\in\Delta(\mathcal{S})$ and hence $V_\gamma:\Delta(\mathcal{S}) \to \mathbb{R}$,  optimization~\eqref{eq:valueiterationpomdp} is infinite-dimensional and we cannot solve it efficiently. 

If the one-step coherent risk measure $\rho_t$ is the total discounted expectation, we can show that optimization problem~\eqref{eq:valueiterationpomdp} simplifies to an infinite-dimensional linear program and equality holds in~\eqref{eq:costpomdpbeliefs}. This can be proved following the same lines as the proof of Corollary~1 but for the belief MDP. Hence, Theorem 2 also provides an optimization based solution to the constrained POMDP problem.

\subsection{Risk-Averse FSC Synthesis via  Policy Iteration}

In order to synthesize risk-averse FSCs, we employ a policy iteration algorithm. Policy iteration incrementally improves a controller by alternating between two steps: Policy Evaluation (computing value functions by fixing the policy) and Policy Improvement (computing the policy by fixing the value functions), until convergence to a satisfactory policy~\cite{bertsekas76}.   
For a risk-averse POMDP, policy evaluation can be carried out by solving~\eqref{eq:valueiterationpomdp}. 
However, as mentioned earlier, ~\eqref{eq:valueiterationpomdp} is difficult to use directly as it must be computed at each (continuous) belief state in the belief space, which is uncountably infinite.

In the following, we show that if instead of considering policies with infinite-memory, we search over finite-memory policies, then we can find suboptimal solutions to~Problem~1 that lower-bound $J_\gamma(\kappa_0)$. To this end, we consider stochastic but finite-memory controllers as described in Section II.C.

Closing the loop around a POMDP with an FSC $\mathcal{G}$ induces a Markov chain. 
% \vspace{0.2cm}
% \begin{defn}[Global Markov Chain] \label{defn:globalMC} {
% Let POMDP $\mathcal{PM}$ have state space $\mathcal{S}$ and let $G$ be the I-states of stochastic finite state controller 
% $\mathcal{G}$. 
The global Markov chain $\mathcal{MC}^{\mathcal{PM},\mathcal{G}}_{\mathcal{S}\times
G}$ (or simply $\mathcal{M C}$, where the stochastic finite state controller and the POMDP are clear from the context) with execution $  \lbrace[s_0,g_0],[s_1,g_1],\dots\rbrace,\ [s_t,\ g_t] \in \mathcal{S}
\times G$. The probability of initial global state $[s_0,g_0]$ is
  \begin{equation*} \label{eq:GlobalMCInitial}
     \iota_{init}\left(\left[s_0,g_0 \right]\right) 
            = \kappa_0(s_0)\kappa(g_0|\kappa_0).
  \end{equation*}
The state transition probability, $T^{\mathcal{M}}$, is given by
  \begin{equation*} \label{eq:GlobalMCTransition}
    \begin{aligned}
      T^{\mathcal{M}} & \left(\left[s_{t+1},g_{t+1}\right] \left|
            \left[s_t,g_t\right] \right. \right)  = \\ 
      \sum_{o\in\mathcal{O}} &
            \sum_{\alpha \in Act}O(o|s_t)\omega(g_{t+1},\alpha |g_t,o)T(s_{t+1}|s_t,\alpha).
    \end{aligned}
  \end{equation*}

\subsection{Risk Value Function Computation}
Under an FSC, the POMDP is transformed into a Markov chain $\mathcal{M}^{\mathcal{PM} \times \mathcal{G}}_{\mathcal{S} \times \mathcal{G}}$ with design probability distributions $\omega$ and $\kappa$. The closed-loop Markov chain $\mathcal{M}^{\mathcal{PM} \times \mathcal{G}}_{\mathcal{S} \times \mathcal{G}}$ is a controlled Markov process with $\{q_t\}=\{[s_t,g_t]\}$, $t=1,2,\ldots$.   In this setting, the total risk functional~\eqref{eq:objrisk} becomes a function of $\iota_{init}$ and FSC $\mathcal{G}$, \textit{i.e.,} 
\begin{multline}\label{eq:objriskfsc}
    J_{\gamma}(\iota_{\mathrm{init}},\mathcal{G}) = \rho^{\gamma} \big( c([s_0,g_0],\alpha_0),c([s_1,g_1],\alpha_1),\ldots \big),\\~~s_0 \sim \kappa_0,~g_0 \sim \kappa,
\end{multline}
where $\alpha_t$s and $g_t$s are drawn from the probability distribution $\omega(g_{t+1},\alpha_t \mid g_t,o_t)$. The constraint functionals $D_{\gamma}^i(\iota_{\mathrm{init}},\mathcal{G})$, $i=1,2,\ldots, n_c$ can also be defined similarly. 

Let $J_\gamma(\boldsymbol{\iota}_{init})$ be the value of Problem 1 under a FSC~$\mathcal{G}$. Then, it is evident that $J_\gamma(\boldsymbol{b}_0) \ge J_\gamma(\boldsymbol{\iota}_{init})$, since FSCs restrict the search space of the policy $\pi$. That is, they can only be as good as the (infinite-dimensional) belief-based policy $\pi(b)$ as $|G|\to \infty$ (infinite-memory).
  
%   \subsection{Example II: Application to Underactuated Bipedal Walking}
%   \input{walkingExample}

\noindent
{\bf Risk Value Function Optimization:} For POMDPs controlled by stochastic finite state controllers, the dynamic program is developed in the global state space
$\mathcal{S}\times G$.  From Theorem~1, we see that for a given FSC, $\mathcal{G}$, and  POMDP $\mathcal{PM}$, the value function $V_{\gamma,\mathcal{M}}([s,g])$ can be computed by solving the following finite dimensional optimization
\begin{align}\label{eq:valueiterationsfc}
&\sup_{\boldsymbol{V}_{\gamma,\mathcal{M}},\boldsymbol{\lambda}\succeq \boldsymbol{0}}~~\langle \boldsymbol{\iota}_{init},\boldsymbol{V}_{\gamma,\mathcal{M}}\rangle - \langle \boldsymbol{\lambda},\boldsymbol{\beta} \rangle \nonumber  \\ &\text{subject to} \nonumber \\
   & V_{\gamma,\mathcal{M}}([s,g]) \le  \sum_{\alpha \in Act}   p(\alpha \mid g) \tilde{c}([s,g],\alpha) \nonumber \\ & \quad \quad \quad + \gamma \sigma\Big \{ V_{\gamma,\mathcal{M}}([s',g']),[s,g], T^{\mathcal{M}}  \left([s',g'] \left|
            [s,g] \right. \right) \Big\}, \nonumber \\
            &\quad \quad \quad \quad \forall s \in \mathcal{S},~\forall g \in G,
\end{align}
where $p(\alpha \mid g) = {\sum_{g' \in \mathcal{G}, o \in \mathcal{O}} \omega(g',\alpha \mid g,o) O(o|g')}, $ and $\tilde{c}([s,g],\alpha)={c}([s,g],\alpha) + \langle \boldsymbol{\lambda}, \boldsymbol{d}([s,g],\alpha)\rangle$. Then, the solution $(\boldsymbol{V}^*_{\gamma,\mathcal{M}},\boldsymbol{\lambda}^*)$ satisfies
\begin{equation} \label{eq:costpomdpbelief2s}
    J_\gamma(\boldsymbol{\iota}_{init}) \ge \langle \boldsymbol{\iota}_{init},\boldsymbol{V}^*_{\gamma,\mathcal{M}}\rangle-\langle\boldsymbol{\lambda}^*,\boldsymbol{\beta}\rangle.
\end{equation} 

Note that since $\rho^\gamma$ is a coherent, Markov risk measure (Assumption 1), $v \mapsto \sigma(v,\cdot,\cdot)$ is convex (because $\sigma$ is also a coherent risk measure). In fact, optimization problem~\eqref{eq:valueiterationsfc} is indeed a DCP in the form of~\eqref{eq:DCP}, where we should replace $V_\gamma$ with $V_{\gamma,\mathcal{M}}$ and set $f_0(\boldsymbol{\lambda})=\langle \boldsymbol{\lambda},\boldsymbol{\beta} \rangle$, $g_0(\boldsymbol{V}_{\gamma,\mathcal{M}})=\langle \boldsymbol{\iota_{init}},\boldsymbol{V}_{\gamma,\mathcal{M}}\rangle$, $f_1({V}_{\gamma,\mathcal{M}})={V}_{\gamma,\mathcal{M}}$, $g_1(\boldsymbol{\lambda})=\sum_{\alpha \in Act}   p(\alpha \mid g) \tilde{c}([s,g],\alpha)$, and $g_2({V}_{\gamma,\mathcal{M}})=\gamma \sigma({V}_{\gamma,\mathcal{M}},\cdot,\cdot)$. 

% Then, we can re-write~\eqref{eq:valueiterationsfc} in minimization form as
% \begin{align}\label{eq:DCP}
%       & \inf_{{V}_{\gamma,\mathcal{M}},\boldsymbol{\lambda} \succeq \boldsymbol{0}}~~ f_0(\boldsymbol{\lambda})-g_0({V}_{\gamma,\mathcal{M}})  \nonumber  \\
%         &\text{subject to} \nonumber  \\
%         &f_1({V}_{\gamma,\mathcal{M}})-g_1(\boldsymbol{\lambda})-g_2({V}_{\gamma,\mathcal{M}}) \le 0, ~~\forall s,g.
% \end{align}
The above optimization is in standard DCP form because  $f_0$ and $g_1$ are convex (linear) functions of $\boldsymbol{\lambda}$ and $g_0$, $f_1$, and $g_2$ are convex functions in ${V}_{\gamma,\mathcal{M}}$.

%  DCPs are found in  many applications, such as  inverse covariance estimation in statistics~\cite{thai2014inverse} and feature selection in  machine learning~\cite{le2008dc}. While DCPs can be globally solved~\cite{horst1999dc}, \textit{e.g.} via branch and bound algorithms~\cite{lawler1966branch}, locally optimal solutions can be obtained based on methods of nonlinear optimization~\cite{Bertsekas99} more efficiently. In fact, we use a variant of the convex-concave procedure~\cite{lipp2016variations} for solving DCPs. In particular, the disciplined convex-concave programming (DCCP)~\cite{shen2016disciplined} method linearizes DCP problems into a (disciplined) convex program (carried out automatically via the DCCP Python package~\cite{shen2016disciplined}), which is then converted into an equivalent cone program by
% replacing each function with its graph implementation. Then, the cone program can be solved readily by available convex programming solvers, such as CVXPY~\cite{diamond2016cvxpy}. 

Solving~\eqref{eq:DCP} gives a set of value functions $V_{\gamma,\mathcal{M}}$. In the next section, we discuss how to use the solutions from this DCP in our proposed policy iteration algorithm to sequentially improve the FSC parameters~$\omega$.

\subsection{I-States Improvement}
Let $\vec{ V}_{\gamma,\mathcal{M}}(g) \in \mathbb{R}^{|S|}$ denote the vectorized $ V_{\gamma,\mathcal{M}}([s,g])$ in $s$. We say that an I-state $g$ is  \emph{improved}, if the tunable FSC parameters associated with that I-state can be adjusted so that $\vec{ V}^*_{\gamma,\mathcal{M}}(g)$ increases. 

To begin with,  we compute the initial I-state by finding the best valued I-state for a given initial belief, \textit{i.e.}, $\kappa(g_{{init}})  =  1$, where $$g_{{init}}  =  \underset{g \in G}{\mbox{argmax}}~ \left \langle \boldsymbol{\iota}_{init}, \vec{V}_{\gamma,\mathcal{M}}(g) \right\rangle.$$

%This step tries to improve each I-state in a round robin fashion by keeping the other I-states the same. 
After this initialization, we search for FSC parameters $\omega$ that result in an improvement.

\noindent
{\bf I-state Improvement Optimization:} Given value functions $V_{\gamma,\mathcal{M}}([s,g])$ for all $s \in \mathcal{S}$ and $g \in G$ and Lagrangian parameters $\boldsymbol{\lambda}$,
for every I-state $g$, we can find FSC parameters $\omega$ that result in an improvement by solving the following~optimization 
\begin{eqnarray}\label{eq:istateimprovmentCOpt}
    &\underset{\epsilon> 0,\omega(g',\alpha|g,o)}{\max} \ \ \ \epsilon& \nonumber \\
    &{\text{subject to}}&  \nonumber \\
    &{\text{Improvement Constraint:}}&  \nonumber\\
    &V_{\gamma,\mathcal{M}}([s,g]) + \epsilon  \le  \text{r.h.s. of \eqref{eq:valueiterationsfc}},~~ \forall s \in \mathcal{S}, & \nonumber \\
&{\text{Probability Constraints:}}&  \nonumber\\
&\underset{(g',\alpha)\in G\times Act}{\sum}\omega(g',\alpha\mid g,o)=1,~~ \forall o \in \mathcal{O},& \nonumber \\
&\omega(g',\alpha \mid g, o)\ge 0,~~\forall g'\in G, \alpha \in Act, o \in \mathcal{O}.&
\end{eqnarray}
% \begin{equation}
% \centering
% \begin{array}{lrcl}
% & \multicolumn{3}{l}{ \underset{\epsilon,\omega(g',\alpha|g,o)}{\max} \ \ \ \epsilon}\\
% \mbox{subject to} & & & \\
% \multicolumn{4}{l}{\mbox{Improvement constraints:}}\\
% & V^{\beta}([s,g]) + \epsilon & \le & r^{\beta}(s)  \+  
% \multicolumn{4}{l}{\beta \underset{s',g',\alpha,o}{\sum}O(o|s)\omega(g',\alpha|g,o)T(s'|s,\alpha)V^{\beta}([s',g'])\ \ \ \forall s }
% \end{array}
% \end{equation}

Note that the above optimization searches for $\omega$ values that improve the I-state value vector $\vec{ V}^*_{\gamma,\mathcal{M}}(g)$ by maximizing the auxiliary decision variable $\epsilon$. 

Optimization problem~\eqref{eq:istateimprovmentCOpt} is in general non-convex. This can be inferred from the fact that, although the first term  in the r.h.s. of~\eqref{eq:valueiterationsfc} is linear in $\omega$, its convexity or concavity is not clear in the $\sigma$ term for a general coherent risk measure. Fortunately, we can prove the following result.

\begin{proposition}
Let $\boldsymbol{V}_{\gamma,\mathcal{M}}$ and $\boldsymbol{\lambda}$ be given. Then, the I-State Improvement Optimization~\eqref{eq:istateimprovmentCOpt} is a linear program for conditional expectation and CVaR risk measures. Furthermore,~\eqref{eq:istateimprovmentCOpt} is a convex optimization for EVaR risk measure.
\end{proposition}
\begin{proof}
We present different forms of the Improvement Constraint in~\eqref{eq:istateimprovmentCOpt} for different risk measures. Note that the rest of the constraints and the cost function are linear in the decision variables $\epsilon$ and $\omega$. The Improvement Constraint in~\eqref{eq:istateimprovmentCOpt} is linear in $\epsilon$. However, its convexity or concavity in $\omega$ changes depending on the risk measure one considers. We recall from the previous section that in the Policy Evaluation step, the quantities for ${\boldsymbol{V}}_{\gamma,\mathcal{M}}$ and ${\boldsymbol{\lambda}} \succeq \boldsymbol{0}$ (for conditional expectation, CVaR, and EVaR measures) and $\zeta$ for (CVaR and EVaR measures) are calculated and therefore fixed here.

For conditional expectation, the Improvement Constraint alters to
\begin{multline}
V_{\gamma,\mathcal{M}}([s,g]) +\epsilon \le  \sum_{\alpha \in Act}   p(\alpha \mid g) \tilde{c}([s,g],\alpha) \\ + \gamma \sum_{s' \in \mathcal{S}, g' \in \mathcal{G}} V_{\gamma,\mathcal{M}}([s',g']) T^{\mathcal{M}}  \left([s',g'] \left|
            [s,g] \right. \right) , \\  ~~\forall s \in \mathcal{S},~\forall g \in G.
\end{multline}
Substituting the expression for $T^\mathcal{M}$, i.e., 
$$
T^{\mathcal{M}}  \left([s',g'] \left|
            [s,g] \right. \right) =\sum_{o\in\mathcal{O}} 
            \sum_{\alpha \in Act}O(o|s)\omega(g',\alpha |g,o)T(s'|s,\alpha),
$$
and $p(\alpha \mid g)$, i.e., 
$$
p(\alpha \mid g) = {\sum_{g' \in \mathcal{G}, o \in \mathcal{O}} \omega(g',\alpha \mid g,o) O(o|g')},
$$
we obtain
\begin{multline}
    V_{\gamma,\mathcal{M}}([s,g]) +\epsilon \le    {\sum_{ \alpha, g', o } \omega(g',\alpha \mid g,o) O(o|g')} \tilde{c}([s,g],\alpha)  \\+ \gamma \sum_{s', g',o,\alpha} V_{\gamma,\mathcal{M}}([s',g']) O(o|s)\omega(g',\alpha |g,o)T(s'|s,\alpha) , \\  ~~\forall s \in \mathcal{S},~\forall g \in G.
\end{multline}
The above expression is linear in $\omega$ as well as $\epsilon$. Hence, I-State Improvement Optimization becomes a linear program for conditional expectation risk measure.

Based on a similar construction, for CVaR measure, the Improvement Constraint changes to 
\begin{multline}
        V_{\gamma,\mathcal{M}}([s,g]) +\epsilon \le  \sum_{\alpha, g', o} \omega(g',\alpha \mid g,o) O(o|g') \tilde{c}([s,g],\alpha)  \\+ \gamma  \bigg\{ \zeta + \frac{1}{\varepsilon} \sum_{g',s'} \left( V_{\gamma,\mathcal{M}}\left([s',g']\right)-\zeta\right)_{+}  T^{\mathcal{M}}([s',g']|[s,g]) \bigg\},\\~~\forall s \in \mathcal{S},~\forall g \in G.
\end{multline}
After substituting the term for $T^\mathcal{M}$, we obtain
\begin{multline} \label{vscsssaa}
        V_{\gamma,\mathcal{M}}([s,g]) +\epsilon \le  \sum_{\alpha,g', o} \omega(g',\alpha \mid g,o) O(o|g') \tilde{c}([s,g],\alpha)  \\+ \gamma \bigg\{ \zeta + \frac{1}{\varepsilon} \sum_{g',s',o,\alpha} \left( V_{\gamma,\mathcal{M}}\left([s',g']\right)-\zeta\right)_{+}   O(o|s) \times \\ \omega(g',\alpha |g,o)T(s'|s,\alpha)  \bigg\},~~\forall s \in \mathcal{S},~\forall g \in G.
\end{multline}
\begin{figure*}[!t]
% ensure that we have normalsize text
\normalsize
% Store the current equation number.
%\setcounter{MYtempeqncnt}{\value{equation}}
% Set the equation number to one less than the one
% desired for the first equation here.
% The value here will have to changed if equations
% are added or removed prior to the place these
% equations are referenced in the main text.
%\setcounter{equation}{5}
\begin{subequations}\label{eq:istateimprovmentCOptapp}
\begin{eqnarray}
    &\underset{\epsilon> 0,\omega(g',\alpha|g,o)}{\max} \ \ \ \langle \boldsymbol{\iota}_{init},V_{\gamma,\mathcal{M}}\rangle - \langle \lambda,\beta \rangle+\epsilon& \nonumber \\
    &{\text{subject to}}&  \nonumber \\
    &{\text{Improvement Constraint:}}&  \nonumber\\
    &V_{\gamma,\mathcal{M}}([s,g]) + \epsilon - \sum_{\alpha, g', o } \omega(g',\alpha \mid g,o) O(o|g') \tilde{c}([s,g],\alpha) \quad \quad \quad \quad \quad \quad \quad \quad \quad \quad \quad \quad \quad \quad \quad \quad \quad \quad \quad \quad \quad \quad \quad \quad &  \\& \quad \quad \quad \quad \quad \quad \quad- \gamma \bigg\{ \zeta + \frac{1}{\varepsilon} \sum_{g',s',o,\alpha} \left( V_{\gamma,\mathcal{M}}\left([s',g']\right)-\zeta\right)_{+}   O(o|s)\omega(g',\alpha |g,o)T(s'|s,\alpha)  \bigg\} \le 0,~~\forall s \in \mathcal{S},~\forall g \in G, & \nonumber \\
&{\text{Probability Constraints:}}&  \nonumber\\
&\underset{(g',\alpha)\in G\times Act}{\sum}\omega(g',\alpha\mid g,o)=1,~~ \forall o \in \mathcal{O},& \nonumber \\
&\omega(g',\alpha \mid g, o)\ge 0,~~\forall g'\in G, \alpha \in Act, o \in \mathcal{O}.&
\end{eqnarray}
\end{subequations}
% Restore the current equation number.
%\setcounter{equation}{\value{MYtempeqncnt}}
% IEEE uses as a separator
\hrulefill
% The spacer can be tweaked to stop underfull vboxes.
\vspace*{4pt}
\end{figure*}

% The second term on the right hand side of the above expression can be re-written as a maximization 
% \begin{multline} \label{eq:dffsfssss}
%         V_{\gamma,\mathcal{M}}([s,g]) +\epsilon \le  \sum_{\alpha \in Act}   {\sum_{g' \in \mathcal{G}, o \in \mathcal{O}} \omega(g',\alpha \mid g,o) O(o|g')} \tilde{c}([s,g],\alpha)  \\+ \gamma\sup_{\zeta \in \mathbb{R}} \bigg\{ -\zeta - \frac{1}{\varepsilon} \sum_{g',s',o,\alpha} \left( V_{\gamma,\mathcal{M}}\left([s',g']\right)-\zeta\right)_{+}   O(o|s)\omega(g',\alpha |g,o)T(s'|s,\alpha)  \bigg\},~~\forall s \in \mathcal{S},~\forall g \in G.
% \end{multline}
% The maximization on the right-hand side of the above inequality can be lumped into the overall maximization in~~\eqref{eq:istateimprovmentCOpt}. 
Furthermore, for fixed $V_{\gamma,\mathcal{M}}$, $\lambda$, and $\zeta$, the above inequality is linear in $\omega$ and $\epsilon$. Hence,~\eqref{vscsssaa} becomes a linear constraint rendering~~\eqref{eq:istateimprovmentCOpt} a linear program (maximizing a linear objective subject to linear constraints), i.e., optimization problem~\eqref{eq:istateimprovmentCOptapp}.

For the EVaR measure, the Improvement Constraint is given by
\begin{multline}
        V_{\gamma,\mathcal{M}}([s,g]) +\epsilon \le  \sum_{\alpha,g', o } \omega(g',\alpha \mid g,o) O(o|g') \tilde{c}([s,g],\alpha)  \\+ \gamma \Bigg\{ \frac{1}{\zeta } \log \left(\frac{\sum_{g',s'}e^{ \zeta {V}_{\gamma,\mathcal{M}}([s',g'])} T^{\mathcal{M}}([s',g']|[s,g])}{\varepsilon} \right)      \Bigg\},\\~~\forall s \in \mathcal{S},~\forall g \in G.
\end{multline}

Substituting the expression for $T^\mathcal{M}$, i.e., 
$$
T^{\mathcal{M}}  \left([s',g'] \left|
            [s,g] \right. \right) =\sum_{o\in\mathcal{O}} 
            \sum_{\alpha \in Act}O(o|s)\omega(g',\alpha |g,o)T(s'|s,\alpha),
$$
we obtain
\begin{multline} \label{eq:cxcvvxvxvx}
    V_{\gamma,\mathcal{M}}([s,g]) + \epsilon \le  \sum\limits_{\alpha, g', o } \omega(g',\alpha \mid g,o) O(o|g') \tilde{c}([s,g],\alpha) \quad \quad \quad \quad \quad \quad \quad \quad \quad \quad \quad \quad \quad \quad \quad \quad \quad \quad \quad \quad \quad \quad \quad \quad   \\ +  \frac{\gamma}{\zeta } \log \left(\frac{\sum\limits_{g',s',o,\alpha}e^{ \zeta {V}_{\gamma,\mathcal{M}}([s',g'])} O(o|s)\omega(g',\alpha |g,o)T(s'|s,\alpha)}{\varepsilon} \right),\\~~\forall s \in \mathcal{S},~\forall g \in G, 
\end{multline}

% The second term on the right hand side of the above expression can be re-written as a maximization, i.e., 
% \begin{multline} \label{eq:xcsdsdfsd}
%         V_{\gamma,\mathcal{M}}([s,g]) +\epsilon \le  \sum_{\alpha \in Act}   {\sum_{g' \in \mathcal{G}, o \in \mathcal{O}} \omega(g',\alpha \mid g,o) O(o|g')} \tilde{c}([s,g],\alpha)  \\+ \gamma\sup_{\zeta >0} \Bigg\{ \frac{-1}{\zeta } \log \left(\frac{\sum_{g',s'}e^{ \zeta {V}_{\gamma,\mathcal{M}}([s',g'])} T^{\mathcal{M}}([s',g']|[s,g])}{\varepsilon} \right)      \Bigg\},~~\forall s \in \mathcal{S},~\forall g \in G.
% \end{multline}
\begin{figure*}[!t]
% ensure that we have normalsize text
\normalsize
% Store the current equation number.
%\setcounter{MYtempeqncnt}{\value{equation}}
% Set the equation number to one less than the one
% desired for the first equation here.
% The value here will have to changed if equations
% are added or removed prior to the place these
% equations are referenced in the main text.
%\setcounter{equation}{5}
\begin{subequations}\label{eq:istateimprovmentCOptfdffgfg}
\begin{eqnarray}
    &\underset{\epsilon> 0,\omega(g',\alpha|g,o)}{\max} \ \ \ \langle \boldsymbol{\iota}_{init},V_{\gamma,\mathcal{M}}\rangle - \langle \lambda,\beta \rangle+\epsilon& \nonumber \\
    &{\text{subject to}}&  \nonumber \\
    &{\text{Improvement Constraint:}}&  \nonumber\\
    &V_{\gamma,\mathcal{M}}([s,g]) + \epsilon - \sum_{\alpha, g', o } \omega(g',\alpha \mid g,o) O(o|g') \tilde{c}([s,g],\alpha) \quad \quad \quad \quad \quad \quad \quad \quad \quad \quad \quad \quad \quad \quad \quad \quad \quad \quad \quad \quad \quad \quad \quad \quad &  \\& \quad \quad \quad \quad \quad \quad \quad- \gamma \Bigg\{ \frac{1}{\zeta } \log \left(\frac{\sum_{g',s',o,\alpha}e^{ \zeta {V}_{\gamma,\mathcal{M}}([s',g'])} O(o|s)\omega(g',\alpha |g,o)T(s'|s,\alpha)}{\varepsilon} \right)      \Bigg\} \le 0,~~\forall s \in \mathcal{S},~\forall g \in G, & \nonumber \\
&{\text{Probability Constraints:}}&  \nonumber\\
&\underset{(g',\alpha)\in G\times Act}{\sum}\omega(g',\alpha\mid g,o)=1,~~ \forall o \in \mathcal{O},& \nonumber \\
&\omega(g',\alpha \mid g, o)\ge 0,~~\forall g'\in G, \alpha \in Act, o \in \mathcal{O}.&
\end{eqnarray}
\end{subequations}
% Restore the current equation number.
%\setcounter{equation}{\value{MYtempeqncnt}}
% IEEE uses as a separator
\hrulefill
% The spacer can be tweaked to stop underfull vboxes.
\vspace*{4pt}
\end{figure*}
In the above inequality, the first term on the right-hand side of the is linear in $\omega$ and the second term on the right-hand side (logarithm term) is concave in $\omega$ (convex if all terms are moved to the left side, since $-\log(x)$ is convex in $x$). Therefore,~\eqref{eq:cxcvvxvxvx} becomes a convex constraint rendering~\eqref{eq:istateimprovmentCOpt} a convex optimization problem (maximizing a linear objective subject to linear and convex constraints) for EVaR measures. That is, the I-State Improvement Optimization takes the convex optimization form of~\eqref{eq:istateimprovmentCOptfdffgfg}.  
\end{proof}

% However, for the sake of coherent risk measures in this paper, $\sigma$ is linear in $\omega$ for conditional expectation and CVaR and concave in $\omega$ for EVaR measures. Thus, optimization~\eqref{eq:istateimprovmentCOpt} converts to a linear program for conditional expectation and CVaR measures and to a convex optimization for the EVaR measure.

If no improvement is achieved by optimization~\eqref{eq:istateimprovmentCOpt}, \textit{i.e.,} $\epsilon = 0$, for fixed number of internal states $|G|$, we can increase  $|G|$  by one following the footsteps of the bounded policy iteration method proposed in~\cite[Section V.B]{ahmadi2020risk}.

\subsection{Policy Iteration Algorithm}

Algorithm \ref{algo:policyiteration} outlines the main steps in the proposed policy iteration method for the constrained risk-averse FSC synthesis. The algorithm has two distinct parts. First, for fixed  parameters of the FSC ($\omega$), policy evaluation is carried out, in which $V_{\gamma,\mathcal{M}}([s,g])$ and $\lambda$ are computed using  DCP~\eqref{eq:valueiterationsfc} (Steps 2, 10 and 18).
% For each I-state $g$, we have the following: 
% \begin{eqnarray}\label{eq:policyevaluationCP}
%     &\underset{\epsilon_1> 0,\epsilon_2> 0,V_{\gamma,\mathcal{M}}}{\min} \ \ \ \epsilon_1-\epsilon_2& \nonumber \\
%     &\pushleft{\text{subject to}}&  \nonumber \\
%     &V_{\gamma,\mathcal{M}}([s,g]) - \big(\text{r.h.s. of \eqref{eq:valueiterationsfc}}\big)    \le  \epsilon_1,~~ \forall s \in \mathcal{S}, & \nonumber \\
%         &V_{\gamma,\mathcal{M}}([s,g]) - \big(\text{r.h.s. of \eqref{eq:valueiterationsfc}}\big)    \ge  \epsilon_2,~~ \forall s \in \mathcal{S}.
% \end{eqnarray}
% In fact, the above optimization solves~\eqref{eq:valueiterationsfc} for  $V_{\gamma,\mathcal{M}}$.
 Second, after evaluating the current value functions and the Lagrange multipliers, an improvement is carried out either by changing the parameters of existing I-states via optimization~\eqref{eq:istateimprovmentCOpt}, or if no new parameters can improve any I-state, then a fixed number of I-states are added to escape the local minima (Steps 14-17) based on the method proposed in~\cite[Section V.B]{ahmadi2020risk}.

\begin{algorithm}[t]
\caption{Policy Iteration For Synthesizing Constrained Risk-Averse FSC}
\label{algo:policyiteration}
\begin{algorithmic}[1]
\REQUIRE (a) An initial feasible FSC, $\mathcal{G}$. (b) Maximum size of FSC $N_{max}$. (c) $N_{new} \le N_{max}$ number of I-states
\STATE $improved \leftarrow True$
\STATE Compute the value vectors, $\vec{V}_{\gamma,\mathcal{M}}$ and Lagrange multipliers $\boldsymbol{\lambda}$, based on  DCP~\eqref{eq:valueiterationsfc}.
\iffalse
\STATE Compute the value vectors $\vec{V}^{av}$ for the average reward criterion $\eta_{av}^{ssd}$ as in  (\eqref{eq:averageBE}), or efficient approximation in Section \ref{sec:efficient2}.
\fi
\WHILE {$|G| \le N_{max}$ \AND $improved = True$}
\STATE $improved \leftarrow False$
\FORALL {I-states $g \in G$}
\STATE  Solve the I-State Improvement Optimization~\eqref{eq:istateimprovmentCOpt}.
\IF {I-State Improvement Optimization results in $\epsilon > 0$}
\STATE Replace the parameters $\omega$ for I-state $g$
\STATE $improved \leftarrow True$
\STATE Compute the value vectors, $\vec{V}_{\gamma,\mathcal{M}}$ and Lagrange multipliers $\boldsymbol{\lambda}$, based on  optimization~\eqref{eq:valueiterationsfc}.
\iffalse
\STATE Compute the value vectors $\vec{V}^{av}$ for the average reward criterion $\eta_{av}^{ssd}$ as in  (\eqref{eq:averageBE}), or efficient approximation in Section \ref{sec:efficient2}.
\fi
\ENDIF
\ENDFOR
\IF {$improved = False$ \AND $|G| < N_{max}$}
\STATE $n_{added} \leftarrow 0$
\STATE $N'_{new} \leftarrow \min(N_{new},N_{max}-|G|)$
\STATE {Try to add $N'_{new}$ I-state(s) to $\mathcal{G}$.}
\STATE $n_{added} \leftarrow $ actual number of I-states added in previous step.
\IF {$n_{added} > 0$}
\STATE $improved \leftarrow True$
\STATE Compute the value vectors, $\vec{V}_{\gamma,\mathcal{M}}$ and Lagrange multipliers $\boldsymbol{\lambda}$, based on  optimization~\eqref{eq:valueiterationsfc}.
\iffalse
\STATE Compute the value vectors $\vec{V}^{av}$ for the average reward criterion $\eta_{av}^{ssd}$ as in  (\eqref{eq:averageBE}), or efficient approximation in Section \ref{sec:efficient2}.
\fi
\ENDIF
\ENDIF
\ENDWHILE
\ENSURE $\mathcal{G}$
\end{algorithmic}
\end{algorithm}

\section{Numerical Experiments}\label{sec:example}

In this section, we evaluate the proposed methodology with numerical experiments. In addition to the traditional total expectation, we consider two other coherent risk measures, namely, CVaR and EVaR. All experiments were carried out on a MacBook Pro with 2.8 GHz Quad-Core Intel Core i5 and 16 GB of RAM. The resultant linear programs and DCPs were solved using CVXPY~\cite{diamond2016cvxpy} with DCCP~\cite{shen2016disciplined} add-on in Python.

\subsection{Rover MDP Example Set Up}

 \begin{figure}[t] \label{fig:gwa_again}\centering{
\includegraphics[scale=.35]{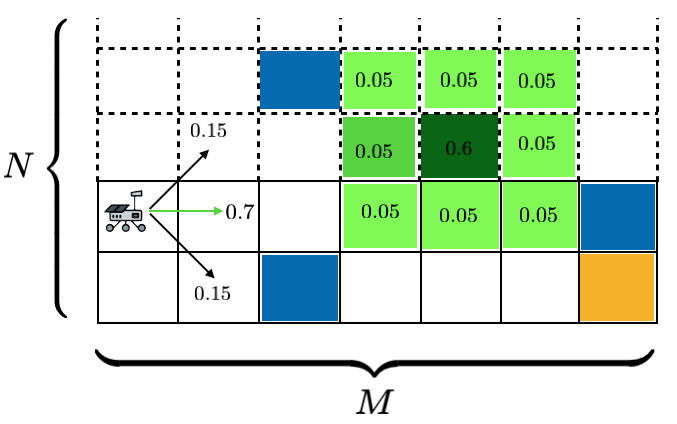}
\vspace{-0.4cm}
\caption{Grid world illustration for the rover navigation example. Blue cells denote the obstacles and the yellow cell denotes the goal.}
\vspace{-.5cm}
} 
 \end{figure}
 
An  agent (e.g. a rover) must autonomously navigate a 2-dimensional terrain map (e.g. Mars surface) represented by an $M \times N$ grid with $0.25 MN$ obstacles. The state space is given by
 $ \mathcal{S} = \{s_{i}|i=x+My,x\in\{1,\dots,M\},y \in \{1,\dots,N\}\}. $ 
The action set available to the robot is $ Act = \{E,\ W,\ N,\ S,\ NE,\ NW,\ SE,\ SW\}$. The state transition probabilities for various cell types are shown for actions $E$ in Figure 2, i.e., the agent moves to the grid implied by the action with $0.7$ probability but can also move to any adjacent ones with $0.3$ probability. Partial observability arises because the rover cannot determine obstacle cell
location from measurements directly. The observation space is
$ \mathcal{O}=\{o_{i} |i=x+My, x\in\{1,\dots,M\},y\in\{1,\dots,N\}\}.$ Once at an adjacent cell to an obstacle, the rover can identify an actual obstacle position (dark green) with probability $0.6$, and a distribution over the nearby cells (light green). 

 Hitting an obstacle incurs the immediate cost  of $10$, while the goal grid region has zero immediate cost. Any other grid has a cost of $2$ to represent fuel consumption. The discount factor is set to $\gamma=0.95$.

The objective is to compute a safe path that is fuel efficient, \textit{i.e.,} solving Problem 1. To this end, we consider total expectation, CVaR, and EVaR as the coherent risk measure. 

Once a policy is calculated, as a robustness test, inspired by~\cite{chow2015risk}, we included a set of single grid obstacles that are perturbed in a random direction to one of the neighboring grid
 cells with probability $0.3$ to represent uncertainty in the terrain map. For each risk measure, we run $100$ Monte Carlo simulations with the calculated policies and count failure rates, \textit{i.e.,} the number of times a collision has occurred during a run.

% CVaR is given by 
% \begin{equation}
%     \rho_t(c_{t+1}) = \inf_{z \in \mathbb{R}} \left\{ z + \frac{1}{\alpha} \mathbb{E}\left[  (c_{t+1}-z)_{+} \mid \mathcal{F}_t    \right]                        \right\},
% \end{equation}
% where $(\cdot)_{+}=\max\{\cdot, 0\}$ and the infimum should be understood point-wise. In general, the confidence level $\alpha$ may be $\mathcal{F}_t$-measurable function with values in the interval $(0,1)$. Here, we assume  $\alpha \in (0,1)$. A value of $\alpha \simeq 1$ corresponds to a risk-neutral policy; whereas, a value of $\alpha \simeq 0$ is rather a risk-averse policy. For CVaR risk measure, \eqref{eq:valueiterationsfc} can be computed as
% \begin{align*}
% \hspace{-.5cm}    V_{\gamma,\mathcal{M}}([s,g]) =& \sum_{\alpha , g' , o }   \omega(g',\alpha \mid g,o) O(o|g') c([s,g],\alpha) \\&+ \gamma\inf_{z \in \mathbb{R}} \bigg\{ z + \frac{1}{\alpha} \sum_{g',s', o,\alpha} \left( V\left([s',g']\right)-z\right)_{+} \\
%     &~~\times O(o\mid s)\omega(g',\alpha \mid g, o) T(s'\mid s,\alpha)\bigg\},
% \end{align*}
% where the infimum on the right hand side of the above equation can either be solved by line search techniques or
% by representation in terms of an elementary linear programming problem since it is convex in $z$~\cite[Theorem 1]{rockafellar2000optimization} (the function $(\cdot)_+$ is increasing and convex~\cite[Lemma A.1., p. 117]{ott2010markov}).

\subsection{ MDP Results}

%  \begin{figure*}[!h] \label{fig:mdp}\centering{
% %\includegraphics[scale=.3]{mdpexample.png}\\
% \includegraphics[scale=.23]{mdpnorisk.eps}
% \includegraphics[scale=.24]{mdpcvar2.eps}
% \includegraphics[scale=.24]{mdpcvar.eps}
% \vspace{-0.5cm}
% \caption{Results for the MDP example with total expectation (left), CVaR (middle), and EVaR (right) coherent risk measures. The goal is located at the yellow cell. Notice the 10 single cell obstacles used for robustness test.}
% %\vspace{-0.5cm}
% } 
%  \end{figure*}

%\subsection{POMDP Results}

To evaluate the technique discussed in Section~IV, we assume that there is no partial observation. In our experiments, we consider four grid-world sizes of $10\times 10$, $15 \times 15$,  $20 \times 20$, and $30 \times 30$ corresponding to $100$, $225$, $400$, and $900$ states, respectively. For each grid-world, we randomly allocate 25\% of the grids to obstacles, including 3, 6, 9, and 12 uncertain (single-cell) obstacles for the $10\times 10$, $15 \times 15$, $20 \times 20$, and $30 \times 30$ grids, respectively.  In each case, we solve DCP~\eqref{eq:valueiteration} (linear program in the case of total expectation) with $|\mathcal{S}||Act|=MN \times 8 = 8MN$ constraints and $MN+2$ variables (the risk value functions $V_\gamma$'s, Langrangian coefficient $\lambda$, and $\zeta$ for CVaR and EVaR). In these experiments, we set $\varepsilon= 0.2$ for CVaR and EVaR coherent risk measures to represent risk-averse policies. The fuel budget (constraint bound $\beta$) was set to 50, 10, 200, and 600 for the $10\times 10$, $15 \times 15$, $20 \times 20$, and $30 \times 30$ grid-worlds, respectively. The initial condition was chosen as $\kappa_0(s_M)=M-1$, \textit{i.e.,} the agent starts at the second left most grid at the bottom.

A summary of our numerical experiments is provided in Table 1. Note the computed values of Problem 1 satisfy $\mathbb{E}(c)\le \mathrm{CVaR}_\varepsilon(c) \le \mathrm{EVaR}_\varepsilon(c)$, which is consistent with that fact that EVaR is a more conservative coherent risk measure than CVaR~\cite{ahmadi2012entropic}. 

\begin{table}[t!]
\label{table:comparison}
\centering
\setlength\tabcolsep{2.5pt}
\begin{tabular}{lccccc} \midrule
\makecell{$(M \times N)_{\rho_t}$}  & \makecell{$J_\gamma(\kappa_0)$}  & \makecell{ Total \\  Time [s] }  & \makecell{\# U.O.}   & \makecell{ F.R.}   \\ 
\midrule
$(10\times 10)_{\mathbb{E}}$           & 9.12                          & 0.8 & 3                & 11\%                                        \\[1.5pt]
$(15\times 15)_{\mathbb{E}}$          & 12.53                          &  0.9  & 6           & 23\%                                             \\[1.5pt]
$(20\times 20)_{\mathbb{E}}$         & 19.93                          &  1.7 & 9               & 33\%                                         \\[1.5pt]
$(30\times 30)_{\mathbb{E}}$         & 27.30                          &  2.4 & 12               & 41\%                                         \\[1.5pt]
\midrule
$(10\times 10)_{\text{CVaR}_{0.7}}$            & $\ge$12.04                          &  5.8 & 3            &8\%                                            \\[1.5pt]
$(15\times 15)_{\text{CVaR}_{0.7}}$           & $\ge$14.83                         & 9.3   & 6           &18\%                                          \\[1.5pt]
$(20\times 20)_{\text{CVaR}_{0.7}}$         & $\ge$20.19                          & 10.34   & 9          &19\%                                           \\[1.5pt]
$(30\times 30)_{\text{CVaR}_{0.7}}$         & $\ge$34.95                          &  14.2 & 12               & 32\%                                         \\[1.5pt]
$(10\times 10)_{\text{CVaR}_{0.2}}$            & $\ge$14.45                          &  6.2 & 3            &3\%                                            \\[1.5pt]
$(15\times 15)_{\text{CVaR}_{0.2}}$           & $\ge$17.82                         & 9.0   & 6           &5\%                                          \\[1.5pt]
$(20\times 20)_{\text{CVaR}_{0.2}}$         & $\ge$25.63                          & 11.1   & 9          &13\%                                           \\[1.5pt]
$(30\times 30)_{\text{CVaR}_{0.2}}$         & $\ge$44.83                          &  15.25 & 12               & 22\%                                         \\[1.5pt]
\midrule
$(10\times 10)_{\text{EVaR}_{0.7}}$            & $\ge$14.53                         & 4.8  & 3          &4\%                                            \\[1.5pt]
$(15\times 15)_{\text{EVaR}_{0.7}}$          & $\ge$16.36                         & 8.8         &  6         &11\%                                      \\[1.5pt]
$(20\times 20)_{\text{EVaR}_{0.7}}$            & $\ge$29.89                         & 10.5      & 9           &15\%                                       \\[1.5pt]
$(30\times 30)_{\text{EVaR}_{0.7}}$         & $\ge$54.13                          &  14.99 & 12               & 12\%                                         \\[1.5pt]
$(10\times 10)_{\text{EVaR}_{0.2}}$            & $\ge$18.03                         & 5.8  & 3          &1\%                                            \\[1.5pt]
$(15\times 15)_{\text{EVaR}_{0.2}}$          & $\ge$21.10                         & 8.7         &  6         &3\%                                      \\[1.5pt]
$(20\times 20)_{\text{EVaR}_{0.2}}$            & $\ge$24.08                         & 10.2      & 9           &7\%                                       \\[1.5pt]
$(30\times 30)_{\text{EVaR}_{0.2}}$         & $\ge$63.04                          &  14.25 & 12               & 10\%                                         \\[1.5pt]
\midrule
\end{tabular}
\caption{Comparison between total expectation, CVaR, and EVaR coherent risk measures. $(M \times N)_{\rho_t}$ denotes experiments with grid-world of size $M \times N$ and one-step coherent risk measure $\rho_t$. $J_\gamma(\kappa_0)$ is the valued of the constrained risk-averse problem (Problem 1). Total Time denotes the time taken by the CVXPY solver to solve the associated linear programs or DCPs in seconds. $\#$ U.O. denotes the number of single grid uncertain obstacles used for robustness test.  F.R. denotes the failure rate out of 100 Monte Carlo simulations with the computed policy. }
\end{table}

  \begin{figure*}[!h] \label{fig:mdp}\centering{
\includegraphics[scale=.23]{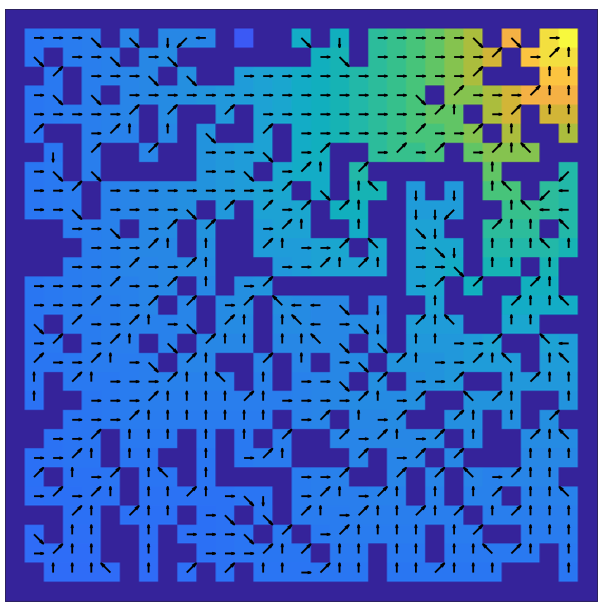}~
\hspace{.5cm}
\includegraphics[scale=.23]{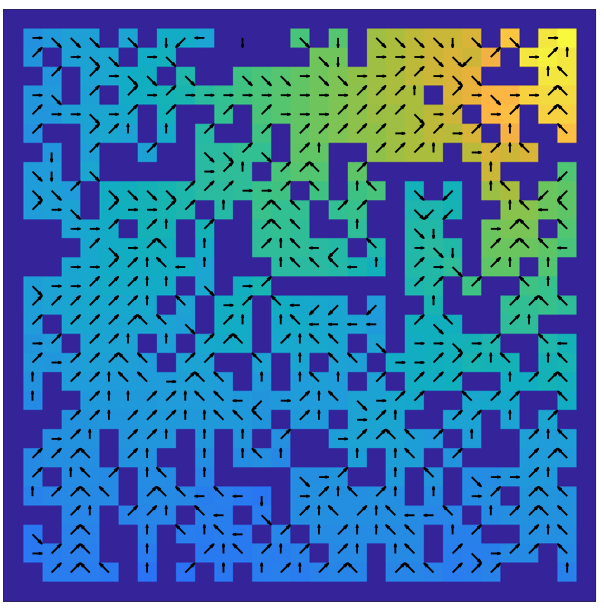}
\hspace{.5cm}
\includegraphics[scale=.23]{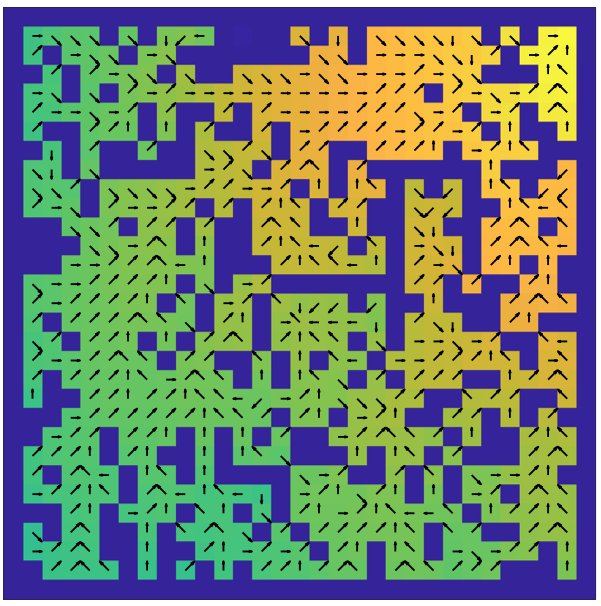}
\vspace{-0cm}
\caption{Results for the MDP example with total expectation (left), CVaR (middle), and EVaR (right) coherent risk measures. The goal is located at the yellow cell. Notice the 9 single cell obstacles used for robustness test.}
%\vspace{-0.5cm}
} 
 \end{figure*}

For total expectation coherent risk measure, the calculations took significantly less time, since they are the result of solving a set of linear programs. For CVaR and EVaR, a set of DCPs were solved. CVaR calculation was the most computationally involved. This observation is consistent with~\cite{ahmadi2019portfolio} were it was discussed that EVaR calculation is much more efficient than CVaR. Note that these calculations can be carried out offline for policy synthesis and then the policy can be applied for risk-averse robot path planning.

The table also outlines the failure ratios of each risk measure. In this case, EVaR outperformed both CVaR and total expectation in terms of robustness, which is consistent with the fact that EVaR is more conservative. In addition, these results imply that, although discounted total expectation is a measure of performance in high number of Monte Carlo simulations, it may not be practical to use it for mission-critical decision making under uncertainty scenarios. CVaR and especially EVaR seem to be a more reliable metric for performance in planning under uncertainty.

For the sake of illustrating the computed policies, Figure~3 depicts the results obtained from solving DCP~\eqref{eq:valueiteration} for a $30\times 30$ grid-world. The arrows on grids depict the (sub)optimal actions and the heat map indicates the values of Problem 1 for each grid state.  Note that the values for EVaR are greater than those for CVaR and the values for CVaR are greater from those of total expectation. This is in accordance with the theory that $\mathbb{E}(c)\le \mathrm{CVaR}_\varepsilon(c) \le \mathrm{EVaR}_\varepsilon(c)$~\cite{ahmadi2012entropic}. In addition, by inspecting the computed actions in obstacle dense areas of the grid-world (for example, the middle right area), we infer that the actions in more risk-averse cases (especially, for EVaR) have a higher tendency to steer the agent away from the obstacles given the diagonal transition uncertainty as depicted in Figure 2; whereas, for total expectation, the actions are merely concerned about reaching the goal.
\vspace{-0.3cm}

\subsection{POMDP Results}

In our experiments, we consider two grid-world sizes of $10\times 10$ and $20 \times 20$ corresponding to $100$ and $400$ states, respectively. For each grid-world, we allocate 25\% of the grid to obstacles, including 8, and 16 uncertain (single-cell) obstacles for the $10\times 10$ and $20 \times 20$ grids, respectively.  In each case, we run Algorithm 1 for risk-averse FSC synthesis with $N_{max}=6$ and a maximum number of 100 iterations were considered. 

% with $|\mathcal{S}||Act|=MN \times 8 = 8MN$ constraints and $MN+2$ variables (the risk value functions $V_\gamma$'s, Langrangian coefficient $\lambda$, and $\zeta$ for CVaR and EVaR).

In these experiments, we set the confidence level $\varepsilon= 0.15$ for CVaR and EVaR coherent risk measures. The fuel budget (constraint bound $\beta$) was set to 50 and 200 for the $10\times 10$ and $20 \times 20$ grid-worlds, respectively. The initial condition was chosen as $\kappa_0(s_M)=1$, \textit{i.e.,} the agent starts at the right most grid at the bottom. 

A summary of our numerical experiments is provided in Table 1. Note the computed values of Problem 1 satisfy $\mathbb{E}(c)\le \mathrm{CVaR}_\varepsilon(c) \le \mathrm{EVaR}_\varepsilon(c)$~\cite{ahmadi2012entropic}. 
% This is consistent with theoretical analysis of~\cite{ahmadi2012entropic} that EVaR is a more conservative coherent risk measure than CVaR. 

\begin{table}[t!]
\label{table:comparison}
\centering
\setlength\tabcolsep{2.5pt}
\begin{tabular}{lccccc} \midrule
\makecell{$(M \times N)_{\rho_t}$}  & \makecell{$J_\gamma(\iota_{init})$}  & \makecell{ AIT  [s] }  & \makecell{\# U.O.}   & \makecell{ F.R.}   \\ 
\midrule
$(10\times 10)_{\mathbb{E}}$           & 10.53                          & 0.2 & 3                & 15\%                                        \\[1.5pt]
$(20\times 20)_{\mathbb{E}}$         & 19.98                          &  0.3 & 9               & 37\%                                         \\[1.5pt]
\midrule
$(10\times 10)_{\text{CVaR}_{0.7}}$            & $\ge$11.02                          &  2.9 & 3            &9\%                                            \\[1.5pt]
$(20\times 20)_{\text{CVaR}_{0.7}}$         & $\ge$20.19                          & 7.5   & 9          &22\%                                           \\[1.5pt]
$(10\times 10)_{\text{CVaR}_{0.2}}$            & $\ge$16.53                          &  3.1 & 3            &4\%                                            \\[1.5pt]
$(20\times 20)_{\text{CVaR}_{0.2}}$         & $\ge$24.92                          & 7.6   & 9          &16\%                                           \\[1.5pt]
\midrule
$(10\times 10)_{\text{EVaR}_{0.7}}$            & $\ge$15.02                         & 3.3  & 3          &5\%                                            \\[1.5pt]
$(20\times 20)_{\text{EVaR}_{0.7}}$            & $\ge$23.42                         & 9.9      & 9           &11\%                                       \\[1.5pt]
$(10\times 10)_{\text{EVaR}_{0.2}}$            & $\ge$19.62                         & 3.9  & 3          &2\%                                            \\[1.5pt]
$(20\times 20)_{\text{EVaR}_{0.2}}$            & $\ge$29.36                         & 9.7      & 9           &6\%                                       \\[1.5pt]
\midrule
\end{tabular}
\caption{Comparison between total expectation, CVaR, and EVaR coherent risk measures. $(M \times N)_{\rho_t}$ denotes experiments with grid-world of size $M \times N$ and one-step coherent risk measure $\rho_t$. $J_\gamma(\iota_{init})$ is the valued of the constrained risk-averse POMDP problem (Problem 1). AIT denotes the average time spent for each iteration of Algorithm 1. $\#$ U.O. denotes the number of single grid uncertain obstacles used for robustness test.  F.R. denotes the failure rate out of 100 Monte Carlo simulations with the computed policy. }
\end{table}

For total expectation coherent risk measure, the calculations took significantly less time, since they are the result of solving a set of linear programs. For CVaR and EVaR, a set of DCPs were solved in the Risk Value Function Computation step. In the I-State Improvement step, a set of linear programs were solved for CVaR and convex optimizations for EVaR. Hence, EVaR calculation was the most computationally involved in this case. 
% This observation is consistent with~\cite{ahmadi2019portfolio} were it was discussed that EVaR calculation is much more efficient than CVaR. 
Note that these calculations can be carried out offline for policy synthesis and then the policy can be applied for risk-averse robot path planning.

The table also outlines the failure ratios of each risk measure. In this case, EVaR outperformed both CVaR and total expectation in terms of robustness, tallying with the fact that EVaR is conservative. In addition, these results suggest that, although discounted total expectation is a measure of performance in high number of Monte Carlo simulations, it may not be practical to use it for real-world planning under uncertainty scenarios. CVaR and especially EVaR seem to be a more reliable metric for performance in planning under uncertainty. 

  \begin{figure}[!t] \label{fig:mdp}\centering{
\includegraphics[scale=.4]{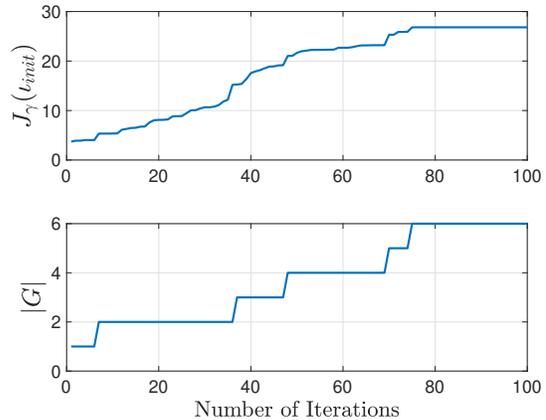}~
\caption{The evolution of the lower-bound and the number of i-states with respect to the number of iterations of Algorithm 1 for the $20\times20$ gridworld and EVaR coherent risk measure.}
%\vspace{-0.5cm}
} 
 \end{figure}
 
   \begin{figure*}[!h] \label{fig:pomdp}\centering{
\includegraphics[scale=.26]{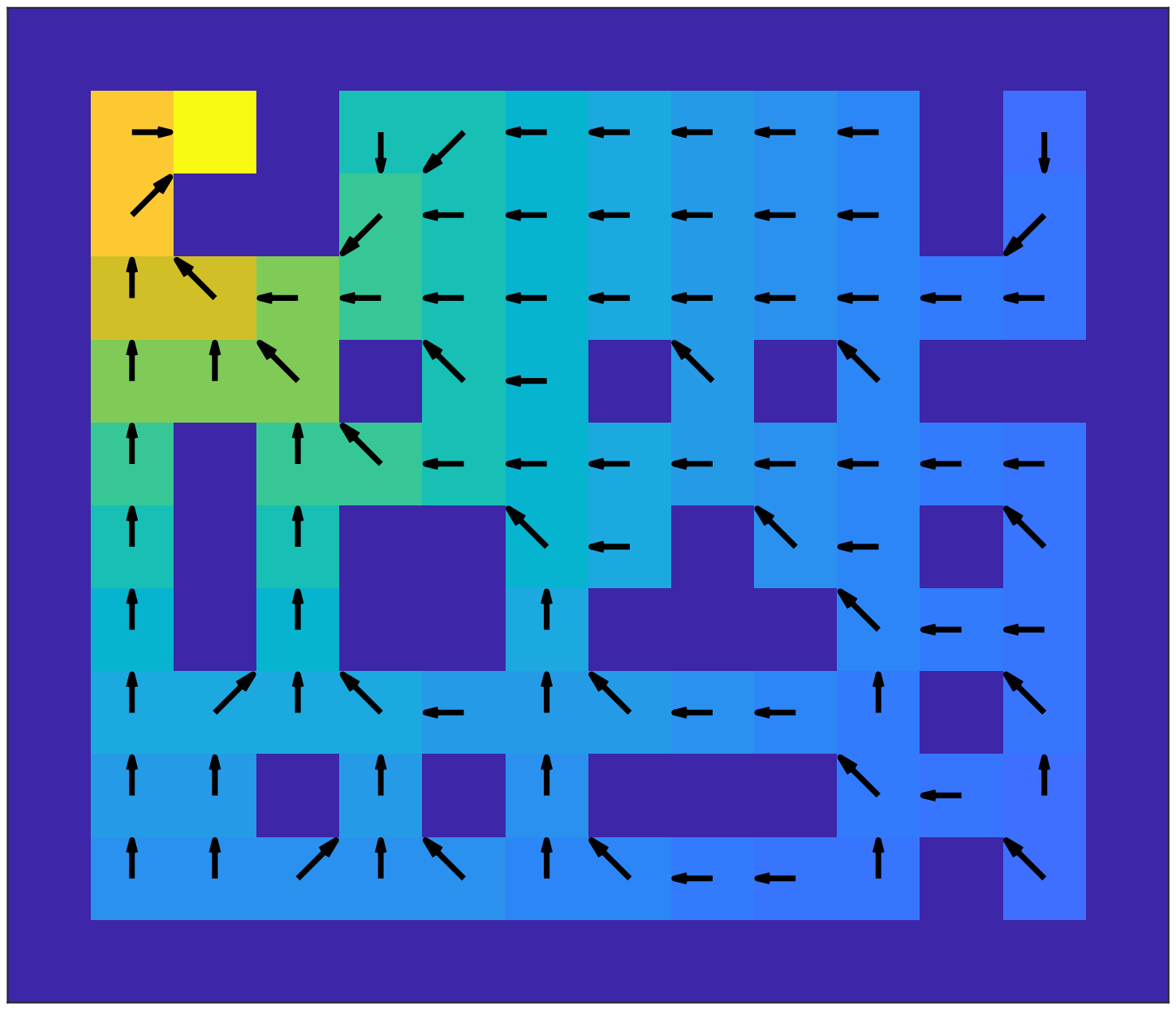}~
\hspace{0cm}
\includegraphics[scale=.26]{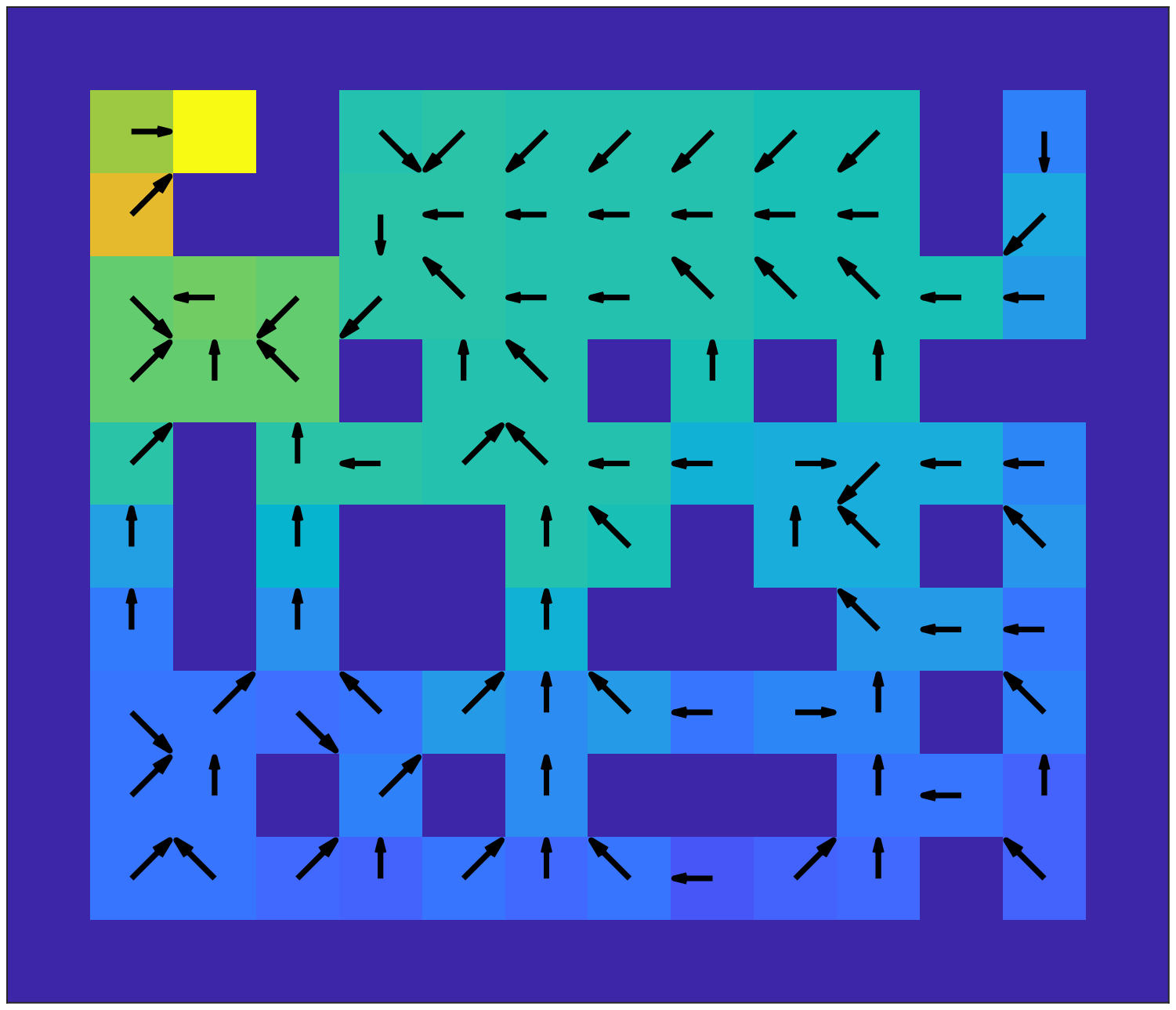}
\hspace{0cm}
\includegraphics[scale=.26]{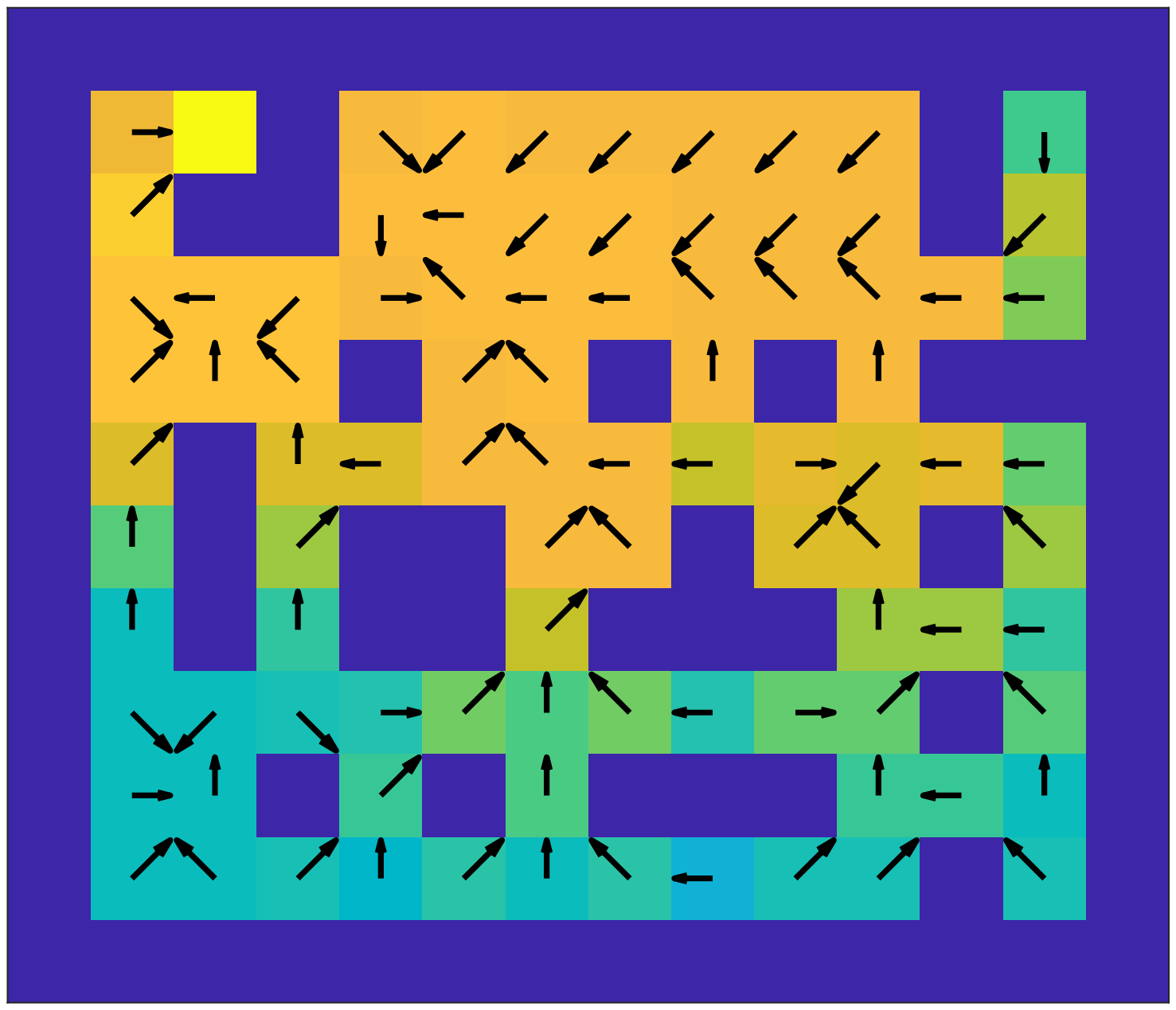}
\vspace{0cm}
\caption{Results for the POMDP example with total expectation (left), CVaR (middle), and EVaR (right) coherent risk measures. The goal is located at the yellow cell. Notice the 9 single cell obstacles used for robustness test.}
%\vspace{-0.5cm}
} 
 \end{figure*}

For the sake of illustrating the computed policies, Figure~3 depicts the results obtained from solving DCP~\eqref{eq:valueiteration} for a $20\times 20$ grid-world. The arrows on grids depict the (sub)optimal actions and the heat map indicates the values of Problem 1 for each grid state.  Note that the values for EVaR are greater than those for CVaR and the values for CVaR are greater from those of total expectation. This is in accordance with the theory that $\mathbb{E}(c)\le \mathrm{CVaR}_\varepsilon(c) \le \mathrm{EVaR}_\varepsilon(c)$~\cite{ahmadi2012entropic}. 

Moreover, for the $20\times20$ gridworld with EVaR coherent risk measure, Figure~2 depicts the evolution of the number of FSC I-states $|G|$ and the lower bound on the optimal value of Problem 1, $J_{\gamma}(\iota_{init})$, with respect to the iteration number of Algorithm 1. We can see that as the number of I-states increase, the lower bound is improved.

% If an improvement is found, \textit{i.e.}, $\epsilon > 0$, the parameters of the I-state are updated by the corresponding maximizing $\omega$. 

% \section{Constrained Risk-Averse POMDPs}\lab

  \section{Conclusions and Future Research}
  
  We proposed an optimization-based method for designing policies for MDPs and POMDPs with coherent risk measure objectives and constraints. We showed that such value function optimizations are in the form of DCPs. In the case of POMDPs, we proposed a policy iteration method for finding sub-optimal FSCs that lower-bound the constrained risk-averse problem and we demonstrated that dependent on the coherent risk measure of interest the policy search can be carried out via a linear program or a convex optimization. Numerical experiments were provided to show the efficacy of our approach. In particular, we showed that considering coherent risk measures lead to significantly lower collision rates in Monte Carlo simulations in navigation problems. 

In this work, we focused on discounted infinite horizon risk-averse problems. Future work will explore other cost criteria~\cite{carpin2016risk}. The interested reader is referred to our preliminary results on total cost risk-averse MDPs~\cite{ahmadi2021risk}, where in Bellman's equations for the risk-averse stochastic shortest path problem are derived. Expanding on the latter work, we will also explore high-level mission specifications in terms of temporal logic formulas for risk-averse MDPs and POMDPs~\cite{ahmadi2020stochasticltl,rosolia2021time}. Another area for more research is concerned with receding-horizon motion planning under uncertainty with coherent risk constraints~\cite{dixit2020risk,hakobyan2019risk}, with particular application in robot exploration in unstructured subterranean environments~\cite{fan2021step} (also see works on receding horizon path planning where the coherent risk measure is in the total cost~\cite{sopasakis2019risk,singh2018framework} rather than the collision avoidance constraint).

\section*{Acknowledgment}

M. Ahmadi acknowledges  stimulating discussions with Dr. Masahiro Ono at NASA Jet Propulsion Laboratory and Prof. Marco Pavone at Nvidia Research-Stanford University.

\footnotesize{
\bibliography{references}
}
 \bibliographystyle{plain}

\end{document}